\documentclass{article}
\usepackage{fullpage}
\usepackage{amsfonts,amssymb,amsmath, amsthm}
\usepackage{bbm}

\usepackage[utf8]{inputenc} 
\usepackage[T1]{fontenc}    
\usepackage{url}            
\usepackage{booktabs}       
\usepackage{nicefrac}       
\usepackage{microtype}      
\usepackage{graphicx}
\usepackage{float}
\usepackage{subfigure}

\usepackage{array} 

\newcommand*{\horzbar}{\rule[.5ex]{2.5ex}{0.5pt}}

\usepackage{thmtools}
\usepackage{thm-restate}
\usepackage{xcolor}

\def\var{\textnormal{var}}
\def\unif{\mbox{\rm\scriptsize unif}}

\def\ind{{\bf 1}}

\def\R{{\mathbb{R}}}
\def\pr{{\rm Pr}}
\def\E{{\mathbb E}}
\def\X{{\mathcal X}}
\def\Y{{\mathcal Y}}
\def\Z{{\mathcal Z}}
\def\A{{\mathcal A}}

\def\F{{\mathcal F}}
\def\G{{\mathcal G}}
\def\Q{{\mathcal Q}}

\def\N{{\mathcal N}}
\def\NN{{\mathbb N}}

\newtheorem{thm}{Theorem}
\newtheorem{lemma}[thm]{Lemma}

\newtheorem{assump}{Assumption}

\DeclareMathOperator*{\argmin}{arg\,min}

\title{Structural query-by-committee}

\author{
Sanjoy Dasgupta \\
 University of California, San Diego \\
{\tt{dasgupta@cs.ucsd.edu}}
\and
Christopher Tosh \\
 University of California, San Diego \\
{\tt{ctosh@cs.ucsd.edu}}
}

\begin{document}

\maketitle

\begin{abstract}
In this work, we describe a framework that unifies many different interactive learning tasks. We present a generalization of the {\it query-by-committee} active learning algorithm for this setting, and we study its consistency and rate of convergence, both theoretically and empirically, with and without noise.
\end{abstract}

\section{Introduction}

We introduce {\it interactive structure learning}, an abstract problem that encompasses many interactive learning tasks that have traditionally been studied in isolation, including active learning of binary classifiers, interactive clustering, interactive embedding, and active learning of structured output predictors. These problems include variants of both supervised and unsupervised tasks, and allow many different types of feedback, from binary labels to must-link/cannot-link constraints to similarity assessments to structured outputs. Despite these surface differences, they conform to a common template that allows them to be fruitfully unified.

In interactive structure learning, there is a space of items $\X$---for instance, an input space on which a classifier is to be learned, or points to cluster, or points to embed in a metric space---and the goal is to learn a {\it structure} on $\X$, chosen from a family $\G$. This set $\G$ could consist, for example, of all linear classifiers on $\X$, or all hierarchical clusterings of $\X$, or all knowledge graphs on $\X$. There is a target structure $g^* \in \G$ and the hope is to get close to this target. This is achieved by combining a loss function or prior on $\G$ with interactive feedback from an expert.

We allow this interaction to be fairly general. In most interactive learning work, the dominant paradigm has been {\it question-answering}: the learner asks a question (like ``what is the label of this point $x$?'') and the expert provides the answer. We allow a more flexible protocol in which the learner provides a constant-sized {\it snapshot} of its current structure and asks whether it is correct (``does the clustering, restricted to these ten points, look right?''). If the snapshot is correct, the expert accepts it; otherwise, the expert fixes some part of it. This type of feedback, first studied in generality in \cite{DL17}, can be called {\it partial correction}. It is a strict generalization of question-answering, and as we explain in more detail below, it allows more intuitive interactions in many scenarios.

In Section~\ref{sec:sqbc}, we present {\it structural query-by-committee}, a simple and general-purpose algorithm that can be used for any instance of interactive structure learning. It is a generalization of the well-known query-by-committee (QBC) algorithm \cite{SOS92,FSST97}, and operates, roughly, by maintaining a posterior distribution over structures and soliciting feedback on snapshots on which there is high uncertainty. We also introduce an adaptation of the algorithm that allows convex loss functions to handle the noise. This helps computational complexity in some practical settings, most notably when $\G$ consists of linear functions, and as we show in Section~\ref{sec:kernels}, also makes it possible to kernelize structural QBC.

In Section~\ref{sec:consistency}, we show that structural QBC is guaranteed to converge to the target $g^*$, even when the expert's feedback is noisy---and in the appendix, we give rates of convergence in terms of a {\it shrinkage} coefficient. In Section~\ref{sec:expts}, we describe experiments using structural QBC for a variety of interactive learning tasks. We end, in Section~\ref{sec:related}, with an overview of related work.

\section{Interactive structure learning}
\label{sec:structure-learning}
The space of possible interactive learning schemes is large and mostly unexplored. We can get a sense of its diversity from a few examples. In {\it active learning}~\cite{S12}, for instance, the goal is to learn a classifier starting from a pool of unlabeled data. The machine adaptively decides which points it wants labeled, and an expert answers these queries as they arise. By focusing on informative points, the machine can often learn a good classifier using far fewer labels than would be needed in a passive setting.

Sometimes, the labels are complex structured objects, such as parse trees for sentences or segmentations of images. In such cases, providing an entire label is time-consuming, and it is easier if the machine simply suggests a label (such as a tree) and lets the expert either accept it or correct some particularly glaring fault in it. We can think of this as interaction with {\it partial correction}. It is more general than the {\it question-answering} usually assumed in active learning, and more convenient in many settings.

Interaction can also be used to augment {\it unsupervised} learning. Despite great improvements in algorithms for clustering, topic modeling, and so on, the outputs of these procedures are rarely perfectly aligned with the user's needs. Complex high-dimensional data can be organized in many different ways: should a collection of animals be clustered according to the Linnaean taxonomy, or their preferred habitats, or how cute they are? These alternatives are all legitimate, and it is impossible for an unsupervised method to magically guess what the user wants. But a modest amount of interaction can potentially overcome this problem of underspecification. For instance, the user can iteratively provide {\tt must-link} and {\tt cannot-link} constraints~\cite{WC00} to edit a flat clustering, or {\it triplet} constraints to edit a hierarchy~\cite{VD16}.

These are just a few examples of the many types of interactive learning that have been investigated. The underlying tasks encompass problems of both supervised and unsupervised learning. The types of feedback range from triplets to partial labels to connectivity constraints. The querying strategies are also rich in variety. Our first goal is to provide a unifying framework in which this profusion of learning problems can be treated. 

\subsection{The space of structures}

Let $\X$ be a set of data points. This could be a pool of unlabeled data to be used for active learning, or a set of points to be clustered, or an instance space on which a metric will be learned, or items on which a knowledge graph is to be constructed.

We wish to learn a {\it structure} on $\X$, chosen from a class $\G$. This could, for instance, be the set of all labelings of $\X$ consistent with a function class $\F$ of classifiers (binary, multiclass, or with complex structured labels), or the set of all partitions of $\X$, or the set of all metrics on $\X$. Of these structures, there is some target $g^* \in \G$ that we wish to attain.

Although interaction will help choose a structure, it is unreasonable to expect that interaction alone could be an adequate basis for this choice. For instance, pinpointing a particular clustering over $n$ points requires $\Omega(n)$ must-link/cannot-link constraints, which is an excessive amount of interaction when $n$ is large.

To bridge this gap, we need a prior or a loss function over structures. For instance, if $\G$ consists of flat $k$-clusterings, then we may prefer clusterings with low $k$-means cost. If $\G$ consists of linear separators, then we may prefer functions with small norm $\|g\|$. In the absence of interaction, the machine would simply pick the structure that optimizes the prior or cost function. In this paper, we assume that this preference is encoded as a prior distribution $\pi$ over $\G$. 

We emphasize that although we have adopted a Bayesian formulation, there is no assumption that the target structure $g^*$ is actually drawn from the prior.

\subsection{Feedback}

We consider schemes in which each individual round of interaction is not expected to take too long. This means, for instance, that the expert cannot be shown an entire clustering, of unrestricted size, and asked to comment upon it. Instead, he or she can only be given a small {\it snapshot} of the clustering, such as its restriction to 10 elements. The feedback on this snapshot will be either be to accept it, or to provide some constraint that fixes part of it.

In order for this approach to work, it is essential that structures be {\it locally checkable}: that is, $g$ corresponds to the target $g^*$ if and only if every snapshot of $g$ is satisfactory. 

When $g$ is a clustering, for instance, the snapshots could be restrictions of $g$ to subsets $S \subseteq \X$ of some fixed size $s$. Technically, it is enough to take $s=2$, which corresponds to asking the user questions of the form `Do you agree with having {\tt zebra} and {\tt giraffe} in the same cluster?'' From the viewpoint of human-computer interaction, it might be preferable to use larger subsets (like $s=5$ or $s=10$), with questions such as ``Do you agree with the clustering $\{\mbox{\tt zebra}, \mbox{\tt giraffe}, \mbox{\tt dolphin}\}, \{\mbox{\tt whale}, \mbox{\tt seal}\}$?'' Larger substructures provide more context and are more likely to contain glaring faults that the user can easily fix ({\tt dolphin} and {\tt whale} must go together). In general, we can only expect the user to provide partial feedback in these cases, rather than fully correcting the substructure.

We now formalize the notion of a snapshot.

\subsection{Snapshots}

Perhaps the simplest type of snapshot of a structure $g$ is {\it the restriction of $g$ to a small number of points}. We start by discussing this case, and later present a more general definition.

\subsubsection{Projections}

For any $g \in \G$ and any subset $S \subseteq \X$ of size $s = O(1)$, let $g |_S$ be a suitable notion of the restriction of $g$ to $S$, which we will sometimes call the {\it projection} of $g$ onto $S$. For instance:
\begin{itemize}
\item $\G$ is a set of classifiers on $\X$.

Then we can take $s = 1$. For any point $x \in \X$, we let $g |_x$ be $(x,g(x))$.

\item $\G$ is a set of partitions (flat clusterings) of $\X$.

For a set $S \subseteq \X$ of size $s \geq 2$, let $g |_S$ be the induced partition on just the points $S$.

\item $\G$ is a set of hierarchical clusterings of $\X$.

For any $s \geq 3$, and any set $S \subseteq \X$ of size $s$, let $g |_S$ be the restriction of the hierarchical clustering $g$ to just the points $S$, that is, the induced hierarchy on $s$ leaves.

\item $\G$ is a set of metrics on $\X$.

For any $s \geq 2$ and any set $S \subseteq \X$ of size $s$, let $g |_S$ denote the $s \times s$ matrix of distances between points in $S$ according to metric $g$.
\end{itemize}
As discussed earlier, from a human-computer interaction point of view, it will often be helpful to pick projections of size larger than the minimal possible $s$. For clusterings, for instance, any $s \geq 2$ satisfies local checkability, but human feedback might be more effective when $s=10$ than when $s=2$.  Thus, in general, the queries made to the expert will consist of snapshots (projections of size $s = 10$, say) that can in turn be decomposed further into {\it atomic units} (projections of size 2).

\subsubsection{Atomic decompositions of structures}

Now we generalize the notion of projection to other types of snapshots and their atomic units.

We will take a functional view of the space of structures $\G$, in which each structure $g$ is specified by its ``answers'' to a set of {\it atomic questions} $\A$. For instance, if $\G$ is the set of partitions of $\X$, then we can take $\A = {\X \choose 2}$, with 
$$ g(\{x,x'\})  
= 
\left\{
\begin{array}{ll}
1 & \mbox{if $g$ places $x,x'$ in the same cluster} \\
0 & \mbox{otherwise}
\end{array} 
\right.
$$ 

The queries made during interaction can, in general, be composed of multiple atomic units, and feedback will be received on at least one of these atoms. Formally, let $\Q$ be the space of queries. In the partitioning example, this might be ${\X \choose 10}$. The relationship between $\Q$ and $\A$ is captured by the following requirements:
\begin{itemize}
\item Each $q \in \Q$ can be decomposed as a set of atomic questions $A(q) \subseteq \A$. In the partitioning example, $A(q)$ is the set of all pairs in $q$.
\item We will overload notation and write $g(q) = \{(a,g(a)): a \in A(q) \}$.
\item The user accepts $g(q)$ if and only if $g$ satisfactorily answers every atomic question in $q$, that is, if and only if $g(a) = g^*(a)$ for all $a \in A(q)$.
\end{itemize}

\noindent To illustrate this notation, we briefly turn to the example of hierarchical clustering.

\subsubsection{Example: hierarchical clustering}

Suppose $\G$ is the space of hierarchical clusterings of $\X$ and the user has in mind a target hierarchy $g^*$. 

A projection $g |_S$, the restriction of hierarchy $g$ to leaves $S$, is correct if and only if it agrees exactly with $g^* |_S$. We can define the atomic questions to be projections of size 3, that is, $\A = {\X \choose 3}$, and view any hierarchy $g$ as a function:
$$ g: \A \rightarrow \{\mbox{rooted trees with three leaves}\} .$$
Note that the hierarchy is fully specified by this function (to make this precise, we need to also fix some canonical ordering of the data points.) The right-hand set can be thought of as a set of possible labels, so that the learning problem resembles multiclass classification. There are four possible rooted trees with leaves $1,2,3$ and thus four labels: 

\begin{center}
\includegraphics[width=3in]{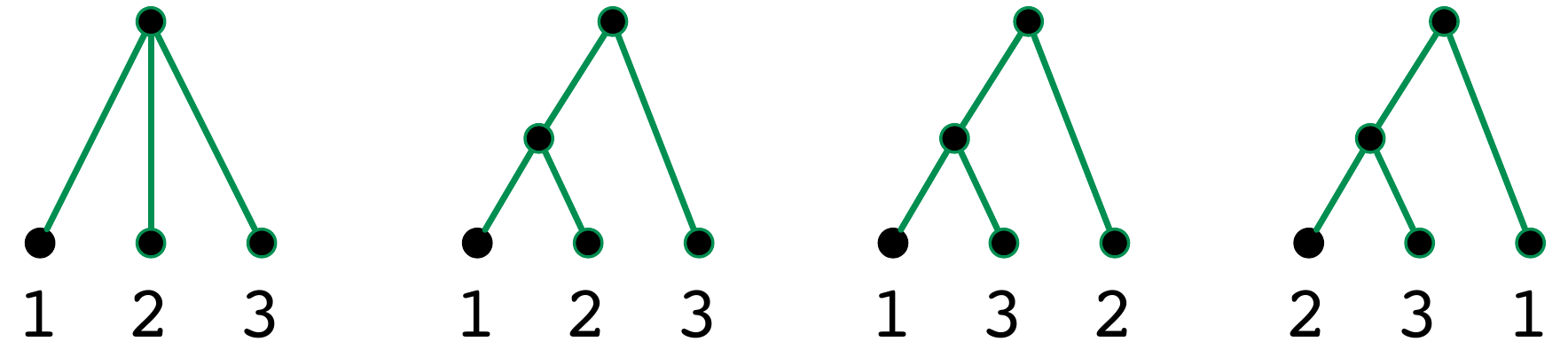}
\end{center}

The queries made by the machine can consist of larger projections, $\Q = {\X \choose s}$ for $s \geq 3$. Each such query $q$ decomposes naturally into its constituent atomic questions: $A(q) = \{a \in \A: a \subseteq q\}$. For instance, if $s = 6$ then $|A(q)| = {6 \choose 3} = 20$.

\subsection{Summary of framework}

%

To summarize, interactive structure learning has two key components:
\begin{itemize}
\item A reduction to multiclass classifier learning.

We view each structure $g \in \G$ as a function on atomic questions $\A$. Thus, learning a good structure is equivalent to picking one whose labels $g(a)$ are correct. 

\item Feedback by partial correction.

For practical reasons we consider broad queries, from a set $\Q$, where each query can be decomposed into atomic questions, allowing for partial corrections. This decomposition is given by the function $A: \Q \rightarrow 2^\A$.
\end{itemize}

The reduction to multiclass classification immediately suggests algorithms that can be used in the interactive setting. We are particular interested in {\it adaptive} querying, with the aim of finding a good structure with minimal interaction. Of the many schemes available for binary classifiers, one that appears to work well in practice and has good statistical properties is {\it query-by-committee}~\cite{SOS92,FSST97}. It is thus a natural candidate to generalize to the broader problem of structure learning.

\section{Structural QBC}
\label{sec:sqbc}

\begin{figure}
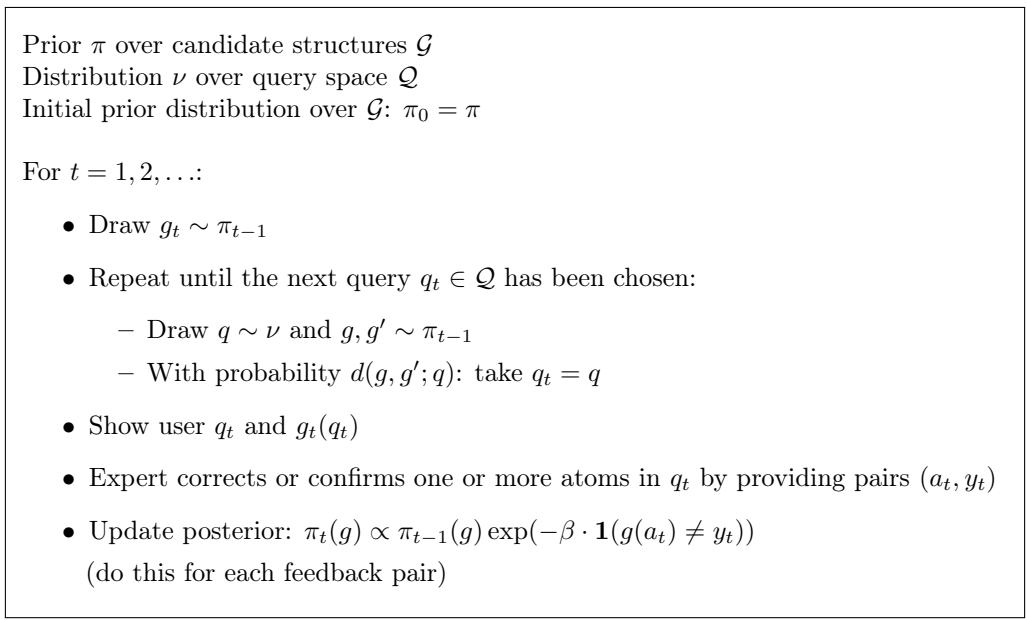

\begin{center}
\fbox{
\begin{minipage}[t]{5.2in}
\vskip.1in

Prior $\pi$ over candidate structures $\G$

Distribution $\nu$ over query space $\Q$

Initial prior distribution over $\G$: $\pi_0 = \pi$
\\

For $t = 1, 2, \ldots$:
\begin{itemize}
\item Draw $g_t \sim \pi_{t-1}$
\item Repeat until the next query $q_t \in \Q$ has been chosen:
\begin{itemize}
\item Draw $q \sim \nu$ and $g,g' \sim \pi_{t-1}$
\item With probability $d(g,g'; q)$: take $q_t = q$
\end{itemize}
\item Show user $q_t$ and $g_t(q_t)$
\item Expert corrects or confirms one or more atoms in $q_t$ by providing pairs $(a_t, y_t)$ 
\item Update posterior: $\pi_{t}(g) \propto \pi_{t-1}(g) \exp(- \beta \cdot {\bf 1}(g(a_t) \neq y_t))$ 

(do this for each feedback pair)
\end{itemize}
\vskip.1in
\end{minipage}}
\end{center}
\caption{Structural QBC for $0-1$ loss.}
\label{fig:structural-QBC}
\end{figure}

Query-by-committee, as originally analyzed in \cite{FSST97}, is an active learning algorithm for binary classification in the noiseless setting. It uses a prior probability distribution $\pi$ over its classifiers and keeps track of the current version space, i.e. the classifiers consistent with the labeled data seen so far. At any given time, the next query is chosen as follows:
\begin{itemize}
	\item Repeat:
	\begin{itemize}
		\item Pick $x \in \X$ at random (e.g. from a pool of unlabeled data)
		\item Pick classifiers $h,h'$ at random from $\pi$ restricted to the current version space
		\item If $h(x) \neq h'(x)$: halt and take $x$ as the point to query
	\end{itemize}
\end{itemize}
In our setting, the feedback at time $t$ is the answer $y_t$ to some atomic question $a_t \in \A$, and we can define the resulting version space to be $\{g \in \G: g(a_{t'}) = y_{t'} \mbox{\ for all $t' \leq t$}\}$. The immediate generalization of QBC would involve picking a query $q \in \Q$ at random, and then choosing it if $g,g'$ sampled from $\pi$ restricted to our version space happen to disagree on it. But this is unlikely to work well, because the answers to queries are no longer binary labels but mini-structures. As a result, $g,g'$ are likely to disagree on minor details even when the version space is quite small, leading to excessive querying. To address this, we will use a more refined notion of the difference between $g(q)$ and $g'(q)$:
\[ d(g,g'; q) \ = \ \frac{1}{|A(q)|} \sum_{a \in A(q)} \ind[g(a) \neq g'(a)]. \]
In words, this is the fraction of atomic subquestions of $q$ on which $g$ and $g'$ disagree. It is a value between 0 and 1, where higher values mean that $g(q)$ differs significantly from $g'(q)$. Then we will query $q$ with probability $d(g,g'; q)$. 

\subsection{Accommodating noisy feedback}

We are interested in the noisy setting, where the user's feedback may occasionally be inconsistent with the target structure. In this case, the notion of a version space is less clear-cut. Our proposed modification is very simple: the feedback at time $t$, say $(a_t,y_t)$, causes the posterior to be updated as follows:
\begin{equation}
\label{eq:0-1-update}
\pi_{t}(g) \ \propto \ \pi_{t-1}(g) \exp(- \beta \cdot {\bf 1}[g(a_t) \neq y_t]) .
\end{equation}
Here $\beta > 0$ is a constant that controls how aggressively errors are punished. In the noiseless setting, we can take $\beta = \infty$ and recover the original QBC update. Even with noise, however, we will demonstrate that this posterior update enjoys convergence guarantees. The full algorithm is shown in Figure~\ref{fig:structural-QBC}. 

\subsection{Uncertainty and informative queries}

What kinds of queries will structural QBC make? To answer this, we first quantify the {\it uncertainty} in the current posterior about a particular query or atom. For $a \in \A$ and $q \in \Q$ and any distribution $\pi$, write
\begin{align*}
u(a; \pi) &= \pr_{g,g' \sim \pi}(g(a) \neq g'(a)) \\
u(q; \pi) &= \E_{g,g' \sim \pi} [d(g,g'; q)]
\ = \  \E_{a \sim \mbox{\rm\scriptsize unif}(A(q))} [u(a; \pi)] ,
\end{align*}
where {\tt unif} denotes the uniform distribution. These uncertainty values lie in the range $[0,1]$. 

The probability that a particular query $q \in \Q$ is chosen in round $t$ by structural QBC is proportional to $\nu(q) u(q; \pi_{t-1})$. Thus, queries with higher uncertainty under the current posterior are more likely to be chosen. As the following lemma demonstrates, getting feedback on uncertain atoms leads to the elimination, or down-weighting in the case of noisy feedback, of many structures inconsistent with $g^*$. 
\begin{lemma}
\label{lem:shrinkage-uncertainty}
Take $\pi$ to be any distribution over $\G$. For any $a \in \A$ and any answer $y$ to $a$, 
\[  \pi(\{g : g(a) \neq y\}) \ \geq \ \frac{1}{2} u(a; \pi).  \]
\end{lemma}
\begin{proof}
Suppose the possible answers to $a$ are $y_1, y_2, \ldots$, and that these have probabilities $p_1 \geq p_2 \geq \cdots$ respectively under $\pi$. That is $p_i = \pi(\{g : g(a) = y_i\}) .$ Then
\[ u(a; \pi) \ = \ 1 - \sum_i p_i^2 . \]
Note that $\pi(\{g : g(a) \neq y \})$ is smallest when $y = y_1$. Thus, we have 
\[ 1 - \pi(\{g : g(a) \neq y \})
\ \leq \ p_1 
\ \leq 
\ \sqrt{1 - u(a; \pi)} \ \leq \ 1 - \frac{1}{2}u(a; \pi), \]
Rearranging gives us the lemma.
\end{proof}

This gives some intuition for the query selection criterion of structural QBC, and will later be used in the proof of consistency.

\subsection{General loss functions}

The update rule for structural QBC, equation~(\ref{eq:0-1-update}), results in a posterior of the form $\pi_t(g) \propto \pi(g) \exp(-\beta \cdot \#(\mbox{mistakes made by $g$}))$, which can in general be difficult to sample from. To address this, we consider a broader class of updates,
\begin{equation}
\label{eq:general-update}
\pi_t(g) \ \propto \ \pi_{t-1}(g) \exp(- \beta \cdot \ell(g(a_t), y_t)) ,
\end{equation}
where $\ell(\cdot, \cdot)$ is a general loss function. In the special case where $\G$ consists of linear functions and $\ell$ is convex, the resulting posterior is a log-concave distribution, which allows for efficient sampling \cite{LV07}. We will show that this update also enjoys nice theoretical properties, albeit under different noise conditions. 

\begin{figure}
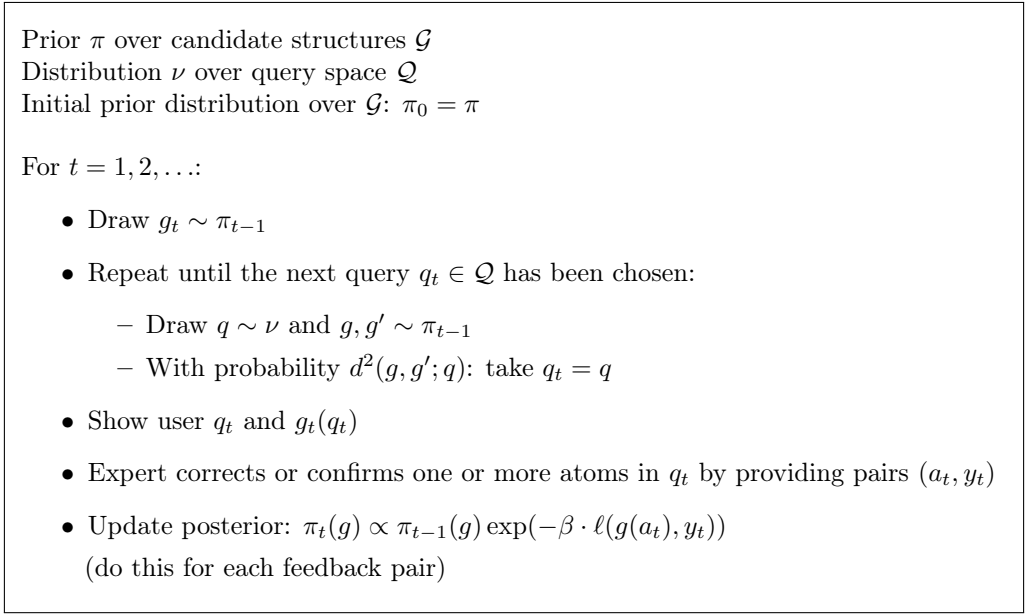

\begin{center}
\fbox{
\begin{minipage}[t]{5.2in}
\vskip.1in

Prior $\pi$ over candidate structures $\G$

Distribution $\nu$ over query space $\Q$

Initial prior distribution over $\G$: $\pi_0 = \pi$
\\

For $t = 1, 2, \ldots$:
\begin{itemize}
\item Draw $g_t \sim \pi_{t-1}$
\item Repeat until the next query $q_t \in \Q$ has been chosen:
\begin{itemize}
\item Draw $q \sim \nu$ and $g,g' \sim \pi_{t-1}$
\item With probability $d^2(g,g'; q)$: take $q_t = q$
\end{itemize}
\item Show user $q_t$ and $g_t(q_t)$
\item Expert corrects or confirms one or more atoms in $q_t$ by providing pairs $(a_t, y_t)$ 
\item Update posterior: $\pi_{t}(g) \propto \pi_{t-1}(g) \exp(- \beta \cdot \ell(g(a_t), y_t))$

(do this for each feedback pair)
\end{itemize}
\vskip.1in
\end{minipage}}
\end{center}
\caption{Structural QBC for general loss functions.}
\label{fig:structural-QBC-general}
\end{figure}

To formally specify the setting, let $\Y$ be the space of answers to atomic questions $\A$, and suppose that structures in $\G$ generate values in some possibly different prediction space $\Z \subseteq \R^d$. That is, we view each $g \in \G$ as a function $g: \A \rightarrow \Z$, and any output $z \in \Z$ gets translated to some prediction in $\Y$. The loss associated with predicting $z$ when the actual answer is $y$ is denoted $\ell(z,y)$. Here are some examples:
\begin{itemize}
\item $0-1$ loss. $\Z = \Y$ and $\ell(z,y) = {\bf 1}(y \neq z)$.
\item Logistic loss. $\Y = \{-1, 1\}$, $\Z = [-B,B]$ for some $B > 0$, and $\ell(z,y) = \ln (1 + e^{-yz})$.
\item Squared loss. $\Y = \{-1,1\}$, $\Z = [-B,B]$, and $\ell(z,y) = (y-z)^2$.
\end{itemize}
When moving from a discrete to a continuous prediction space, it becomes very possible that the predictions, on a particular atomic question, made by two randomly chosen structures will be close but not perfectly aligned. Thus, instead of checking strict equality of these predictions, we need to modify our querying strategy to take into account the distance between them. To this end, we will use the normalized average squared Euclidean distance:
\[ d^2(g, g' ; q) \ = \ \frac{1}{|A(q)|} \sum_{a \in A(q)} \frac{\|g(a) - g'(a)\|^2}{D} \] 
where $D = \max_{a \in \A} \max_{g, g' \in \G} \| g(a) - g'(a)\|^2$. Note that $d^2(g, g' ; q)$ is a value between 0 and 1, and thus we can treat it as a probability, similar to how we used $d(g,g';q)$ in the 0-1 loss setting. The full algorithm is shown in Figure~\ref{fig:structural-QBC-general}.

In the 0-1 loss setting, we saw that structural QBC chooses queries proportional to their uncertainty. What queries will structural QBC make in the general loss setting? Define the variance of $a \in \A$ under distribution $\pi$ as \[ \var(a; \pi) = \sum_{g \in \G} \pi(g) \, \| g(a) - \E_{g' \sim \pi}[(g'(a))] \|^2 = \frac{1}{2} \sum_{g, g' \in \G} \pi(g) \, \pi(g') \, \| g(a) - g'(a) \|^2 \]
and define the variance of a query $q \in \Q$ as the average variance of its constituent atoms,
\[ \var(q; \pi) = \E_{a \sim \unif(A(q))}[\var(a; \pi)] = \frac{1}{|A(q)|} \sum_{a \in A(q)} \var(a; \pi). \]
Then it is not hard to see that the probability that structural QBC chooses $q \in \Q$ at step $t$ is proportional to $\nu(q) \var(q; \pi_{t-1})$. 

\section{Kernelizing structural QBC}
\label{sec:kernels}
{Consider the special case where $\G$ consists of linear functions, i.e. $\G = \{ g_w(x) = \langle x, w \rangle \ : \ w \in \R^d \}$. As mentioned above, when the loss function is convex, the posteriors we encounter are log-concave, and thus efficiently samplable. But what if we want a more expressive class than linear functions? To address this, we resort to kernels. 

Gilad-Bachrach et al. \cite{GNT05} investigated the use of kernels in QBC. In particular, they observed that to run QBC, we need not actually sample from the prior restricted to the current version space. Rather, given a candidate query $x$, it is enough to be able to sample from the distribution this posterior induces over the labelings of $x$. Although their work was in the realizable binary setting, this observation readily applies to our noisy structural setting. 

Let $\phi: \X \rightarrow \R^d$ be a \emph{feature mapping}. Given a prior $\pi$ over $\R^d$, the posterior after observing $(x_1, y_1), \cdots, (x_t, y_t)$ becomes
\[ \pi_t(g_w) \propto  \pi(g_w) \exp \left(- \beta \sum_{i=1}^t \ell(\langle \phi(x_i),  w \rangle, y_i)\right). \]
A particularly interesting case is when $\ell(\cdot, \cdot)$ is the squared-loss and $\pi$ is Gaussian. In this case, we will show that the predictions of the posterior are distributed according to a univariate Gaussian distribution with efficiently computable mean and variance. To show this, we first observe that the posterior is a multivariate Gaussian.
\begin{restatable}{lemma}{GaussianPosteriorLemma}
\label{lem:gaussian-posterior}
Suppose $\pi = \N(0, \sigma_o^2 I_d)$, $\ell(\cdot, \cdot)$ is the squared-loss, and we have observed $(x_1, y_1), \cdots, (x_t, y_t)$. If take $\Phi \in \R^{t \times d}$ to denote the matrix 
\[ \Phi = \left[ \begin{array}{ccc}
   \horzbar & \phi(x_1) & \horzbar \\
   \horzbar & \phi(x_2) & \horzbar	\\
   				  &   \vdots	 & 				\\
   	\horzbar & \phi(x_t) & \horzbar
	\end{array} \right] . \] 
then $\pi_t$ is the multivariate Gaussian $\N(\widehat{\mu}, \widehat{\Sigma})$ where $\widehat{\Sigma} = \left(2 \beta \Phi^T \Phi + \frac{1}{\sigma_o^2} I_d \right)^{-1}$ and $\widehat{\mu} = 2\beta \widehat{\Sigma} \Phi^T y$.
\end{restatable}
We defer the proof of Lemma~\ref{lem:gaussian-posterior} to the appendix. Since $\pi_t$ is a multivariate Gaussian, we know that if $w \sim \pi_t$ and $v \in \R^d$ then $\langle w, v \rangle$ is distributed according to $\N(\mu, \sigma^2)$ where
\begin{align*}
\mu &= v^T \widehat{\mu} = 2\beta v^T \widehat{\Sigma} \Phi^T y \\
\sigma^2 &= v^T \widehat{\Sigma} v = v^T \left(2 \beta \Phi^T \Phi + \frac{1}{\sigma_o^2} I_d \right)^{-1} v
\end{align*}
Unfortunately, directly computing $\mu$ and $\sigma^2$ in the forms above requires expanding out the feature mappings, which is undesirable. However, the following theorem, known as the Woodbury Matrix Identity~\cite[Exercise 13.9]{H02}, allows us to rewrite these terms in a form only involving inner products of the feature vectors.
\begin{thm}[Woodbury Matrix Identity]
\label{thm: Woodbury matrix identity}
Let $T, W, U, V$ be matrices of the appropriate sizes. Then
\[ (T + U W^{-1} V)^{-1} = T^{-1} - T^{-1}U(W + V T^{-1} U)^{-1} V T^{-1}. \]
\end{thm}
\noindent Theorem~\ref{thm: Woodbury matrix identity} implies that we can rewrite $\widehat{\Sigma}$ as
\begin{align*}
\widehat{\Sigma} &= \left(\frac{1}{\sigma_o^2} I_d + \Phi^T (2\beta I_n) \Phi  \right)^{-1} \\
&= \sigma_o^2 I_d - \sigma_o^2 I_d \Phi^T\left( \frac{1}{2\beta} I_n + \Phi(\sigma_o^2 I_d)\Phi^T \right)^{-1} \Phi (\sigma_o^2 I_d)\\
&= \sigma_o^2 \left( I_d - \Phi^T \left( \frac{1}{ 2\sigma_o^2 \beta} I_n + \Phi \Phi^T \right)^{-1} \Phi \right) \\
&=  \sigma_o^2 \left( I_d - \Phi^T \Sigma_0 \Phi \right)
\end{align*}
where $\Sigma_o = \left( \frac{1}{ 2\sigma_o^2 \beta} I_t + \Phi \Phi^T \right)^{-1}$. With this observation in hand, the following lemma readily follows.
\begin{restatable}{lemma}{KernelSquaredLossLemma}
\label{lem:kernel-derivation}
Suppose the assumptions of Lemma~\ref{lem:gaussian-posterior} hold. If $g_w \sim \pi_t$, then $\langle w, \phi(x) \rangle$ is distributed according to $\N(\mu, \sigma^2)$ where
\begin{align*}
\mu &= 2\sigma_o^2 \beta \kappa^T   \left(  I_t  -  \Sigma_o K \right) y \\
\sigma^2 &= \sigma_o^2 \left( \phi(x)^T  \phi(x) -  \kappa^T \Sigma_o \kappa \right)
\end{align*}
where $K_{ij} = \langle \phi(x_i), \phi(x_j)\rangle$, $\kappa_i = \langle \phi(x_i), \phi(x)\rangle$, and $\Sigma_o = \left( \frac{1}{ 2\sigma^2 \beta} I_t + K \right)^{-1}$.
\end{restatable}
The important observation here is that all the quantities involving the feature mapping in Lemma~\ref{lem:kernel-derivation} are inner products. Thus we never need to explicitly construct any feature vectors.}

\section{Consistency of structural QBC}
\label{sec:consistency}

In this section, we look at a typical setting in which there is a finite but possibly very large pool of candidate questions $\Q$, and thus the space of structures $\G$ is effectively finite. Let $g^* \in \G$ be the target structure, as before. Our goal in this setting is to demonstrate the \emph{consistency} of structural QBC, meaning that \[ \lim_{t \rightarrow \infty} \pi_t(g^*) = 1 \] almost surely. To do so, we first formalize our setting. Note that the random outcomes during time step $t$ of structural QBC consist of:
\begin{itemize}
\item the query $q_t$;
\item the atomic question $a_t \in A(q_t)$ that the expert chooses to answer (pick one at random if the expert answers several of them); and
\item the response $y_t$ to $a_t$.
\end{itemize}
Let $\F_t$ denote the sigma-field of all outcomes up to, and including, time $t$. We begin with the special case of structural QBC under the 0-1 loss.

\subsection{Consistency under 0-1 loss}

In order to prove consistency, we will have to make some assumptions about the feedback we receive from a user. For any query $q \in \Q$ and any atomic question $a \in A(q)$, let $\eta(y | a, q)$ denote the conditional probability that the user answers $y$ to atomic question $a$, in the context of query $q$. Our first assumption is that the single most likely answer is $g^*(a)$.
\begin{assump}
\label{assump:0-1-noise}
There exists $0 < \lambda \leq 1$ such that $\eta(g^*(a) |  a,q) - \eta(y | a,q) \geq \lambda$ for all $q \in \Q$ and $a \in A(q)$ and all $y \neq g^*(a)$.
\end{assump}
(We will use the convention $\lambda = 1$ for the noiseless setting.) In the learning literature, Assumption~\ref{assump:0-1-noise} is known as Massart's bounded noise condition~\cite{ABHU15}. As an example, suppose that there are 11 possible answers to an atom. Then a user that answers correctly with probability 0.10 and provides every other incorrect answer with probability 0.09 would satisfy Assumption~\ref{assump:0-1-noise} with $\lambda = 0.01$. Thus, Assumption~\ref{assump:0-1-noise} allows for users who are prone to mistakes, but not inherently biased to a particular incorrect answer.

The following lemma demonstrates that under Assumption~\ref{assump:0-1-noise}, the posterior probability of $g^*$ increases in expectation with each query, as long as the $\beta$ parameter of the update rule in equation~(\ref{eq:0-1-update}) is small enough relative to $\lambda$.
\begin{lemma}
Fix any $t$, and suppose the expert provides an answer to atomic question $a_t \in A(q_t)$ at time $t$. Let $\gamma_t = \pi_{t-1}(\{g \in \G: g(a_t) = g^*(a_t)\})$. Define $\Delta_t$ by:
\[ \E \left[\frac{1}{\pi_{t}(g^*)} \bigg\vert \F_{t-1}, q_t, a_t \right] \ = \ (1- \Delta_t) \frac{1}{\pi_{t-1}(g^*)}, \]
Under Assumption~\ref{assump:0-1-noise}, $\Delta_t$ can be lower-bounded as follows:
\begin{enumerate}
\item[(a)] If $\lambda = 1$ (noiseless setting), $\Delta_t \geq (1-\gamma_t)(1-e^{-\beta})$.
\item[(b)] For any $0 < \lambda \leq 1$, if $\beta \leq \lambda/2$, then $\Delta_t \geq \beta \lambda (1-\gamma_t)/2$.
\end{enumerate}
\label{lem:0-1-noise}
\end{lemma}
\begin{proof}
Let $y_1, y_2, \ldots$ denote the possible answers to $a_t$, and set $p_j = \eta(y_j | a_t, q_t)$ be the probability that the labeler answers $y_j$. Without loss of generality, suppose $p_1 \geq p_2 \geq \cdots$, so that (under Assumption~\ref{assump:0-1-noise}) $g^*(a_t) = y_1$ and $p_1 - p_2 \geq \lambda$.

Further, define $\G_j = \{g \in \G: g(a_t) = y_j\}$. Thus $\gamma_t = \pi_{t-1}(\G_1)$. By averaging over the expert's possible responses, we have
\begin{align*}
\lefteqn{\E \left[\frac{1}{\pi_t(g^*)} \bigg\vert \F_{t-1}, q_t, a_t \right]} \\
&= p_1 \frac{\pi_{t-1}(\G_1) + e^{-\beta}(1 - \pi_{t-1}(\G_1))}{\pi_{t-1}(g^*)} + \sum_{j > 1} p_j \frac{\pi_{t-1}(\G_j) + e^{-\beta}(1 - \pi_{t-1}(\G_j))}{e^{-\beta}\pi_{t-1}(g^*)} \\
&= \frac{1}{\pi_{t-1}(g^*)} \left( p_1 (\pi_{t-1}(\G_1) + e^{-\beta}(1-\pi_{t-1}(\G_1)) + \sum_{j>1} p_j (e^\beta \pi_{t-1}(\G_j) + (1-\pi_{t-1}(\G_{j}))) \right) \\
&= \frac{1}{\pi_{t-1}(g^*)} \left( p_1 (1 - (1-\pi_{t-1}(\G_1))(1-e^{-\beta})) + \sum_{j>1} p_j (1 + (e^\beta-1) \pi_{t-1}(\G_j)) \right) \\
&=  \frac{1}{\pi_{t-1}(g^*)} - \frac{1}{\pi_{t-1}(g^*)}\left( p_1(1-\pi_{t-1}(\G_1))(1-e^{-\beta}) - \sum_{j>1} p_j (e^\beta-1)\pi_{t-1}(\G_j) \right).
\end{align*}
Setting the parenthesized term to $\Delta_t$, we have
$$
\Delta_t 
\ \geq \ 
p_1 (1-\pi_{t-1}(\G_1))(1-e^{-\beta}) - p_2 (e^{\beta}-1) \sum_{j>2} \pi_{t-1}(\G_j) 
\ = \ 
\left( p_1(1-e^{-\beta}) -  p_2(e^{\beta}-1) \right) (1- \gamma_t) .
$$
This yields (a) in the lemma statement. For (b), using the inequalities $e^{\beta} \leq 1 + \beta + \beta^2$ and $e^{-\beta} \leq 1 - \beta + \beta^2$ for $0 \leq \beta \leq 1$, we get
$$
\Delta_t 
\ \geq \ 
\left(  p_1 (\beta - \beta^2)  -  p_2 (\beta + \beta^2) \right) (1- \gamma_t) 
\ \geq \ 
\beta \left( (p_1 - p_2) - \beta( p_1 + p_2 )\right)  (1- \gamma_t) 
\ \geq \ 
\beta (\lambda - \beta)  (1- \gamma_t) .
$$
Taking $\beta \leq \lambda/2$ completes the proof.
\end{proof}

To understand the requirement $\beta = O(\lambda)$, consider an atomic question on which there are just two possible labels, 1 and 2, and the expert chooses these with probabilities $p_1$ and $p_2$, respectively. If the correct answer according to $g^*$ is 1, then $p_1 \geq p_2 + \lambda$ under Assumption~\ref{assump:0-1-noise}. Let $\G_2$ denote structures that answer 2.
\begin{itemize}
\item With probability $p_1$, the expert answers 1, and the posterior mass of $\G_2$ is effectively multiplied by $e^{-\beta}$.
\item With probability $p_2$, the expert answers 2, and the posterior mass of $\G_2$ is effectively multiplied by $e^\beta$.
\end{itemize}
The second outcome is clearly undesirable. In order for it to be counteracted, in expectation, by the first, $\beta$ must be kept fairly small relative to $p_1/p_2$. The condition $\beta \leq \lambda/2$ is sufficient for this.

Thus, Lemma~\ref{lem:0-1-noise} asserts that structural QBC shrinks $1/\pi_t(g^*)$, in expectation, on every round: it corresponds to a random walk with a drift in the right direction. This drift is proportional to $\beta \lambda (1-\gamma_t)$, where $\gamma_t$ is the probability mass, under the current posterior, of structures that agree with $g^*$ on the atom $a_t$. 

Lemma~\ref{lem:0-1-noise} does not, in itself, imply consistency. It is quite possible for $1/\pi_t(g^*)$ to keep shrinking but not converge to 1. Imagine, for instance, that the input space has two parts to it, and we keep improving on one of them but not the other. What we need is, first, to ensure that the queries $q_t$ capture some portion of the uncertainty in the current posterior, and second, that the user chooses an atom that is at least slightly informative. The first condition is assured by the SQBC querying strategy. For the second, we need an assumption.

\begin{assump}
\label{assump:correction-feedback}
There is some minimum probability $p_o > 0$ for which the following holds. If the user is presented with a query $q$ and a structure $g \in \G$ such that $g(q) \neq g^*(q)$, the user will provide feedback on some $a \in A(q)$ such that $g(a) \neq g^*(a)$ with probability at least $p_o$.
\end{assump}
To understand this, note that the interface allows the user to either correct a mistake in $g(q)$ or to corroborate part of it that is correct. The assumption asserts that the former occurs at least some fraction of the time, in situations where $g(q)$ is not perfect. It is one way of avoiding scenarios in which a user never provides feedback on a particular atom $a$. In such a pathological case, we might not be able to recover $g^*(a)$, and thus our posterior will always put some probability mass on structures that disagree with $g^*$ on $a$.

The following lemma, whose proof is deferred to the appendix, gives lower bounds on the quantity $1-\gamma_t$ under Assumption~\ref{assump:correction-feedback}. 
\begin{restatable}{lemma}{CorrectionFeedbackShrinkage}
\label{lem:correction-feedback-shrinkage}
Suppose that $\G$ is finite and the user's feedback obeys Assumption~\ref{assump:correction-feedback}. Then there exists a constant $c>0$ such that for every round $t$ \[ \E[1- \gamma_t \, | \, \F_{t-1}] \ \geq \ c \, \pi_{t-1}(g^*)^2 (1-\pi_{t-1}(g^*))^2 \] where $\gamma_t = \pi_{t-1}(\{g \in \G: g(a_t) = g^*(a_t)\})$, $a_t$ is the atom the user provides feedback on, and the expectation is taken over the randomness in structural QBC and the user's response.
\end{restatable}

Together, Lemmas~\ref{lem:0-1-noise} and~\ref{lem:correction-feedback-shrinkage} show that the sequence $1/\pi_t(g^*)$ is a positive supermartingale that decreases in expectation at each round by an amount that depends on $\pi_t(g^*)$. The following lemma gives us a condition under which such stochastic processes can be guaranteed to converge to 1.

\begin{lemma}
\label{lem:general-consistency}
Suppose that there exists a continuous, non-negative function $f:[0,1] \rightarrow \R_{\geq 0}$ such that $f(1) = 0$ and $f(x) > 0$ for all $x \in (0,1)$. If for each $t \in \NN$, we have \[ \E\left[ \frac{1}{\pi_{t}(g^*)} \, \bigg| \, \F_{t-1} \right] \leq \frac{1}{\pi_{t-1}(g^*)} - f(\pi_{t-1}(g^*))\] 
then $\pi_t(g^*) \rightarrow 1$ almost surely.
\end{lemma}
\begin{proof}
Let $X_t = \pi_t(g^*)$. By assumption, $\frac{1}{X_t}$ is a positive supermartingale, which implies $\lim_{t \rightarrow \infty} \frac{1}{X_t} = \frac{1}{X}$ exists and is finite with probability one. By the Continuous Mapping Theorem, this implies $\lim_{t \rightarrow \infty} X_t = X$ and is non-zero with probability one. On the other hand, we have by the law of total expectation
\[ 1 \leq \E \left[ \frac{1}{X_T} \right] \leq \frac{1}{\pi(g^*)} - \sum_{t=0}^{T-1} \E[f(X_t)]. \]
for all $T \in \NN$, which implies that $\lim_{t\rightarrow \infty} \E[f(X_t)] = 0$.  By Fatou's lemma and the Continuous Mapping Theorem, we have 
\[ 0 = \lim_{t\rightarrow \infty} \E[f(X_t)] = \E \left[\lim_{t\rightarrow \infty} f(X_t) \right] = \E \left[f(X) \right] .\]
Thus, $f(X) = 0$ with probability one. Since $f$ has only two potential zeros at 0 and 1, and since $X > 0$ with probability one, we conclude that $X = 1$ with probability one.
\end{proof}

\noindent As an immediate corollary, we have that structural QBC is consistent.
\begin{thm}
\label{thm:0-1-consistency}
Suppose that $\G$ is finite, and Assumptions~\ref{assump:0-1-noise} and~\ref{assump:correction-feedback} hold. Then if structural QBC is run with a prior distribution $\pi$ in which $\pi(g^*) > 0$, we have $\lim_{t\rightarrow\infty}\pi_t(g^*) = 1$ almost surely.
\end{thm}
\begin{proof}
Combining Lemmas~\ref{lem:0-1-noise} and \ref{lem:correction-feedback-shrinkage}, we have
\begin{align*}
\E \left[\frac{1}{\pi_{t}(g^*)} \bigg\vert \F_{t-1}\right] &=  (1- \E[\Delta_t| \F_{t-1}] )\frac{1}{\pi_{t-1}(g^*)} \\
&\leq \left( 1 - \frac{c\beta \lambda \pi_{t-1}(g^*)^2 (1-\pi_{t-1}(g^*))^2}{2} \right)\frac{1}{\pi_{t-1}(g^*)} \\
&= \frac{1}{\pi_{t-1}(g^*)} -  \frac{c\beta \lambda \pi_{t-1}(g^*) (1-\pi_{t-1}(g^*))^2}{2} 
\end{align*}
Now $f(x) = \frac{c\beta \lambda x (1-x)^2}{2}$ meets all of the conditions of Lemma~\ref{lem:general-consistency}, which concludes the proof.
\end{proof}

In the appendix we provide rates of convergence.

\subsection{Consistency under general losses}

We now turn to analyzing structural QBC with general losses. As before, we will need to make some assumptions. The first is that the loss function is well-behaved.
\begin{assump}
The loss function is bounded, $0 \leq \ell(z,y) \leq B$, and Lipschitz in its first argument,$$\ell(z,y) - \ell(z',y) \leq C \|z - z'\| ,$$ 
for some constants $B, C > 0$.
\label{assump:loss}
\end{assump}
\noindent It is easily checked that this assumption holds for the three loss functions we mentioned earlier.

In the case of 0-1 loss, we assumed that for any atomic question $a$, the correct answer $g^*(a)$ would be given with higher probability than any incorrect answer. We now formulate an analogous assumption for the case of more general loss functions. Recall that $\eta(\cdot | a)$ is the conditional probability distribution over the user's answers to $a \in \A$ (we can also allow $\eta$ to also depend upon the context $q$, as we did before; here we drop the dependence for notational convenience). The expected loss incurred by $z \in \Z$ on this question is thus
\[ L(z,a) = \sum_y \eta(y | a) \, \ell(z,y) . \]
We will require that for any atomic question $a$, this expected loss is minimized when $z = g^*(a)$, and predicting any other $z$ results in excess expected loss that grows with the distance between $z$ and $g^*(a)$.
\begin{assump}
There exists a constant $\lambda > 0$ such that for any atomic question $a \in \A$ and any $z \in \Z$,
$$L(z,a) - L(g^*(a),a) \ \geq \ \lambda \|z - g^*(a)\|^2 .$$
\label{assump:margin}
\end{assump}
\noindent Let's look at what this assumption implies in some concrete settings.
\begin{itemize}
\item $0-1$ loss with $\Y = \Z = \{0,1\}$.
For any $z \in \{0,1\}$, we have
$$ L(z,a) 
= 
\sum_y \eta(y|a) \ell(z,y)
=
1 - \eta(z|a)
$$
and thus Assumption~\ref{assump:margin} is equivalent to requiring that for all $a \in \A$ and for $z \neq g^*(a)$,
$$ \eta(g^*(a)|a) - \eta(z|a) \geq \lambda .$$
This is identical to the earlier Assumption~\ref{assump:0-1-noise}.
\item Squared loss with $\Y = \{-1,1\}$ and $\Z \subset \R$.
Assumption~\ref{assump:margin} requires that for any $a \in \A$, 
$$ g^*(a) 
= \argmin_z L(z,a) 
= \argmin_z \sum_y \eta(y|a) (z-y)^2
= \argmin_z \E[(z-y)^2 | a]
= \E [y | a],
$$
where the expectation is over the choice of $y$ given $a$. If this holds, then for any $z$, by a standard bias-variance decomposition, $L(z,a) - L(g^*(a),a) = (z - g^*(a))^2$, so that $\lambda = 1$.
\item Logistic loss with $\Y = \{-1,1\}$ and $\Z = [-B,B]$. Fix any $a \in A$, and write $p = \eta(1|a)$. Then 
$$ L(z,a) 
=
\sum_y \eta(y|a) \ell(z,y)
= p \ln (1 + e^{-z}) + (1-p) \ln (1 + e^z) .$$
This is minimized by $g^*(a) = \ln p - \ln (1-p)$, and $L(g^*(a), a)$ is then the entropy of a coin with bias $p$. Now pick any other value of $z$, and define $q = 1/(1+e^{-z})$ so that $z = \ln q - \ln (1-q)$. A further calculation shows that $L(z,a) - L(g^*(a),a)$ is exactly the KL divergence $K(p,q)$. Using Pinsker's inequality to lower-bound this in terms of $(p-q)^2$, we then find that Assumption~\ref{assump:margin} is satisfied with 
$$\lambda = 2 \left( \frac{e^B}{(1 + e^B)^2} \right)^2 .$$
\end{itemize}
From these examples, it is clear that requiring $g^*(a)$ to be the minimizer of $L(z,a)$ is plausible if $\Z$ is a discrete space but much less so if $\Z$ is continuous. In general, we can only hope that this holds approximately. With this caveat in mind, we stick with Assumption~\ref{assump:margin} as a useful but idealized mathematical abstraction.

Given these notions, and recalling our earlier definitions of $\var(a; \pi)$ and $\var(q; \pi)$, we have the following general loss analogue of Lemma~\ref{lem:0-1-noise}.
\begin{lemma}
Suppose Assumptions~\ref{assump:loss} and \ref{assump:margin} hold. Take $\beta \leq \min(\lambda/(2C^2), 1/B)$. Suppose that at time $t$, the user provides feedback on an atomic question $a_t \in A(q_t)$ for which $\var(a_t; \pi_{t-1}) \geq \gamma$. Then
$$ \E \left[\frac{1}{\pi_t(g^*)} \bigg\vert \F_{t-1}, q_t, a_t \right] \ = \ (1- \Delta_t) \frac{1}{\pi_{t-1}(g^*)},$$
where $\Delta_t \geq \beta \gamma \lambda/2$.
\label{lem:general-noise}
\end{lemma}

\begin{proof}
Plugging in the update rule for the posterior distribution, we have
\begin{align*}
\E \left[ \frac{1}{\pi_t(g^*)} \bigg| \F_{t-1}, q_t, a_t \right]
&=
\sum_y \eta(y | a_t) \frac{\sum_{g \in \G} \pi_{t-1}(g) \exp(-\beta \cdot \ell(g(a_t),y))}{\pi_{t-1}(g^*) \exp(-\beta\cdot \ell(g^*(a_t),y))} \\
&=
\frac{1}{\pi_{t-1}(g^*)} \sum_y \eta(y | a_t) \sum_g \pi_{t-1}(g) \exp(-\beta\, [\ell(g(a_t),y)-\ell(g^*(a_t),y)]) .
\end{align*} 
To turn this expression into the form $(1-\Delta_t)\frac{1}{\pi_{t-1}(g^*)}$, define
$$ \Delta_t = \sum_y \eta(y | a_t) \sum_g \pi_{t-1}(g) (1 - \exp(-\beta\, [\ell(g(a_t),y)-\ell(g^*(a_t),y)])) .$$
Since $\beta \leq 1/B$, and losses lie in $[0,B]$, we have that $\beta [\ell(g(a_t),y)-\ell(g^*(a_t),y)]$ lies in the range $[-1,1]$. Using the inequality $e^x \leq 1 + x + x^2$ for $-1 \leq x \leq 1$, we then have
\begin{align*}
\Delta_t 
&\geq  
\sum_g \pi_{t-1}(g) \sum_y \eta(y | a_t) ( \beta [\ell(g(a_t),y)-\ell(g^*(a_t),y)] - \beta^2 [\ell(g(a_t),y)-\ell(g^*(a_t),y)]^2 ) \\
&= 
\sum_g \pi_{t-1}(g) \left( \beta [L(g(a_t), a_t) - L(g^*(a_t),a_t)] - \beta^2 \sum_y \eta(y | a_t) [\ell(g(a_t),y)-\ell(g^*(a_t),y)]^2 \right) \\
&\geq
\sum_g \pi_{t-1}(g) \left( \beta \lambda \|g(a_t) - g^*(a_t)\|^2 - \beta^2 \sum_y \eta(y | a_t) C^2 \|g(a_t)-g^*(a_t)\|^2 \right) \\
&=
(\beta \lambda - \beta^2 C^2) \sum_g \pi_{t-1}(g) \, \|g(a_t) - g^*(a_t)\|^2 \\
&\geq
(\beta \lambda - \beta^2 C^2) \mbox{var}(a_t; \pi_{t-1})
\ \geq \ 
\beta(\lambda - \beta C^2)\gamma ,
\end{align*}
where the second and third lines have used Assumptions~\ref{assump:loss} and \ref{assump:margin}.
\end{proof}
Similarly, we can also give a general loss analogue of Lemma~\ref{lem:shrinkage-posterior-mass}.
\begin{restatable}{lemma}{VarShrinkage}
\label{lem:var-shrinkage}
Suppose $\G$ is finite and Assumption~\ref{assump:correction-feedback} holds. Then there exists a constant $c>0$ such that for any round $t$ 
\[\E[\var(a_t; \pi_{t-1}) \, | \, \F_{t-1}] \geq c \, \pi_{t-1}(g^*)^3 (1-\pi_{t-1}(g^*))^2 \]
where $a_t$ is the atom the user provides feedback on and the expectation is taken over both the randomness of user's response and the randomness of structural QBC.
\end{restatable}

The proof of Lemma~\ref{lem:var-shrinkage} is deferred to the appendix. With Lemmas~\ref{lem:general-noise} and \ref{lem:var-shrinkage} in hand, we get the consistency of general-loss structural QBC as a corollary.

\begin{thm}
\label{thm:general-consistency}
Suppose $\G$ is finite and the user's feedback satisfies Assumptions~\ref{assump:correction-feedback}, \ref{assump:loss}, and~\ref{assump:margin}. If the general loss version of structural QBC is run with a prior distribution $\pi$ in which $\pi(g^*) > 0$, then $\lim_{t\rightarrow\infty}\pi_t(g^*) = 1$ almost surely.
\end{thm}

\section{Experiments}
\label{sec:expts}

We now turn to our experiments with structural QBC in a variety of applications. Before we do so, we first consider a way to speed up the query selection procedure.

\subsection{Reducing the randomness in structural QBC}

\begin{figure}
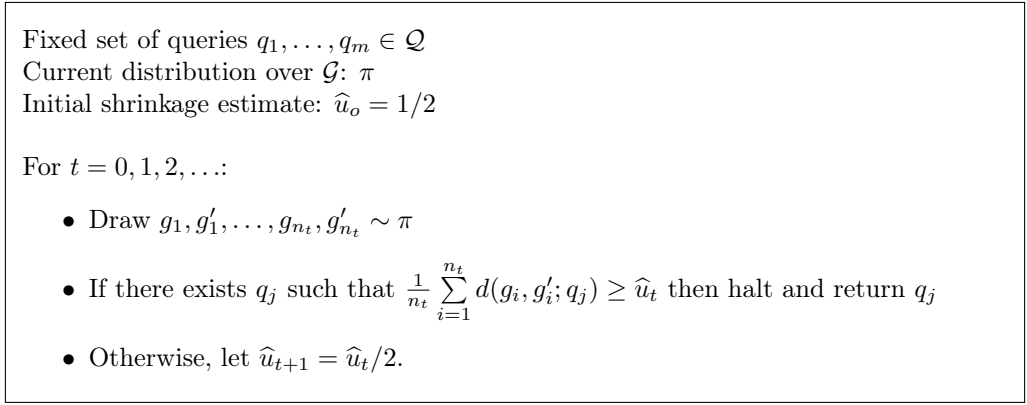

\begin{center}
\fbox{
\begin{minipage}[t]{5.2in}
\vskip.1in

Fixed set of queries $q_1, \ldots, q_m \in \Q$

Current distribution over $\G$: $\pi$

Initial shrinkage estimate: $\widehat{u}_o = 1/2$
\\

For $t = 0, 1, 2, \ldots$:
\begin{itemize}
\item Draw $g_1, g'_1, \ldots, g_{n_t}, g'_{n_t} \sim \pi$
\item If there exists $q_j$ such that $\frac{1}{n_t} \sum \limits_{i=1}^{n_t} d(g_i, g'_i; q_j) \geq \widehat{u}_t$
 then halt and return $q_j$
\item Otherwise, let $\widehat{u}_{t+1} = \widehat{u}_t/2$.
\end{itemize}
\vskip.1in
\end{minipage}}
\end{center}
\caption{Robust query selection for structural QBC, under 0-1 loss.}
\label{fig:robust-query}
\end{figure}

It is easy to see that the query selection procedure of structural QBC is a rejection sampler where each query $q$ is chosen with probability proportional to $\nu(q) u(q; \pi_t)$ (in the case of the zero-one loss) or $\nu(q) \var(q; \pi_t)$ (for general losses). However, without knowing the normalization constant, it is possible for the rejection rate to be quite high, even when there are many queries that have much higher uncertainty or variance than the rest. To circumvent this issue, we introduce a `robust' version of structural QBC, wherein many candidate queries are sampled, and the query that has the highest uncertainty or variance is chosen.


In the zero-one loss case, we can estimate the uncertainty of a candidate query $q$ by first drawing many pairs $g_1, g'_1, \ldots, g_{n}, g'_{n} \sim \pi_t$ and then using the unbiased estimator
\[ \widehat{u}(q;\pi_t) := \frac{1}{n} \sum_{i=1}^n d(g_i, g'_i; q). \]
By Hoeffding's inequality, this quantity concentrates tightly around the true uncertainty of $q$. Unfortunately, the number of structures we need to sample in order to identify the best candidate depends on the uncertainty of that candidate, which we do not know a priori. To circumvent this difficulty, we can start with an optimistic estimate of the largest uncertainty and then iteratively halve our estimate until we are confident that we have found a query with at least that much uncertainty. If the appropriate number of structures are sampled at each round, then it can be shown that this procedure terminates with a candidate query whose uncertainty is within a constant factor of the highest uncertainty~(this is very similar to Lemma 5 of \cite{TD17}). The full algorithm for the zero-one loss case is shown in Figure~\ref{fig:robust-query}.

In experiments, we have found that simply drawing a large number of candidate structures and choosing the query with the highest empirical uncertainty over these works quite well.

\subsection{Clustering}
\begin{figure}[t!]
    \centering
    \begin{subfigure}
        \centering
        \includegraphics[width=5.5in]{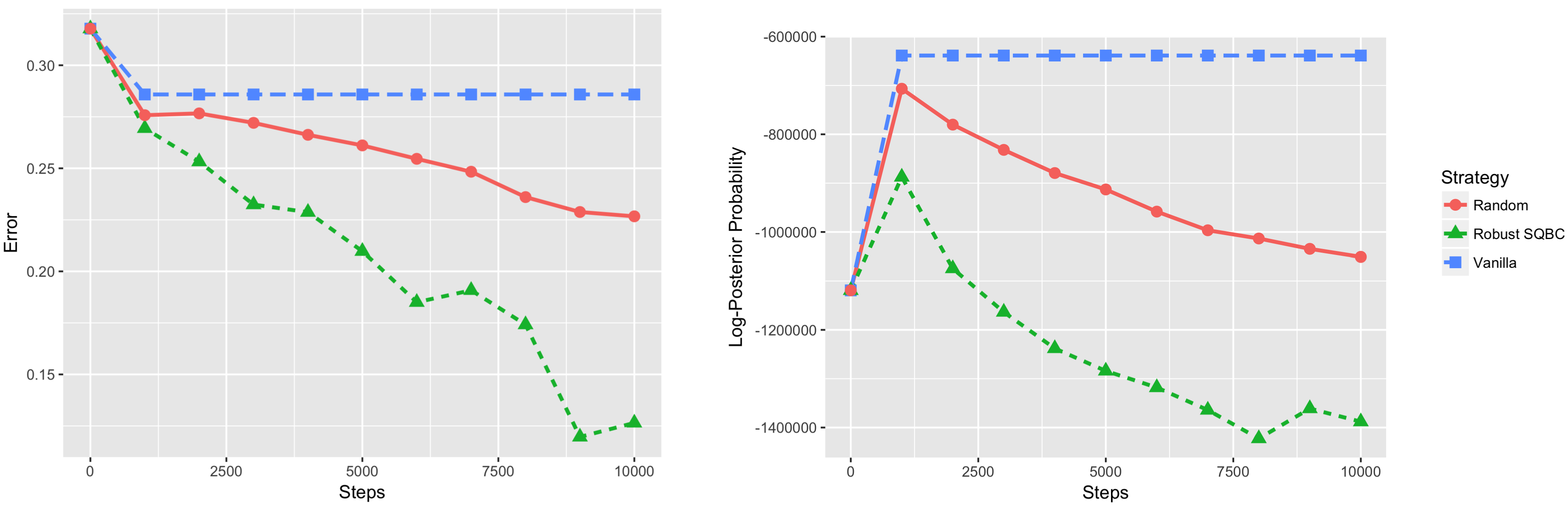} 
        \caption{Mixture of Gaussians experiments on {\tt wine} data set. The x-axis corresponds to full passes of the Gibbs sampler. The hyperparameters used were $\alpha=1$, $\sigma=2$, $\sigma_0 = 4$. \label{fig: wine}}
        \vspace{1.5em}
    \end{subfigure}%
    \begin{subfigure}
        \centering
        \includegraphics[width=5.5in]{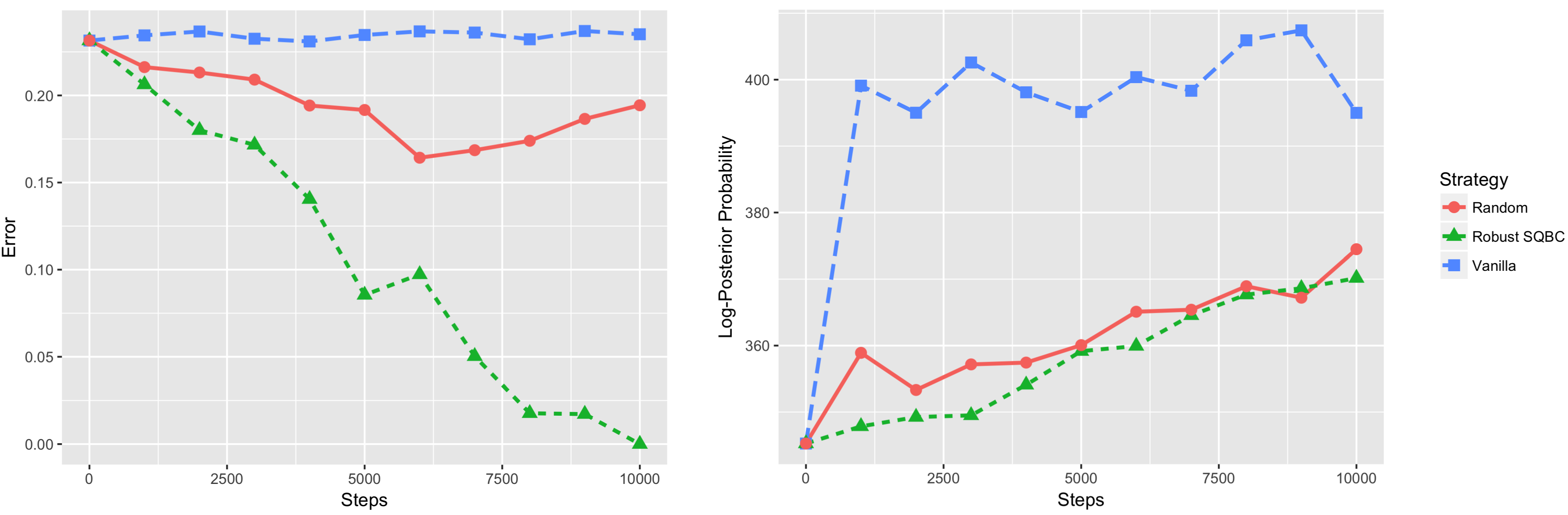} 
	\caption{Mixture of Gaussians experiments on {\tt iris} data set. The x-axis corresponds to full passes of the Gibbs sampler. The hyperparameters used were $\alpha=1$, $\sigma=1$, $\sigma_0 = 2$. \label{fig: iris}}
    \end{subfigure}
\end{figure}

In the case of flat clustering, queries can be taken to be restrictions of the current clustering restricted to a subset of the data points, and feedback can consist of must-link/cannot-link constraints over those data points. In each of our simulations, we ran the robust version of structured QBC where the candidate queries are subsets of size ten and feedback consists of a random correction to the proposed subset clustering. Below we describe the clustering models we used.

\paragraph{Mixture of Gaussians} Consider the following Bayesian generative model for a mixture of $k$ spherical Gaussians. 
\begin{itemize}
	\item Weight vector $(w_1, \ldots, w_k)$ is drawn from a symmetric Dirichlet distribution with parameter $\alpha > 0$
	\item Means $\mu^{(1)}, \ldots, \mu^{(k)} \in \R^d$ are drawn i.i.d. from $\N(\mu_o, \sigma_o^2 I_d)$
	\item For each data point $i=1,\ldots,n$:
	\begin{itemize}
		\item $z_i \in \{ 1, \ldots, k\}$ is drawn from $\mbox{Categorical}(w_1, \ldots, w_k)$
		\item $x^{(i)} \in \R^d$ is drawn from $\N(\mu_{z_i}, \sigma^2 I_d)$
	\end{itemize}
\end{itemize}
Here, $x^{(1)}, \ldots, x^{(n)}$ are the observed data points; $\alpha$, $\mu_o$, $\sigma_o^2$, and $\sigma^2$ are known hyper-parameters; and the $w$'s, $\mu$'s, and $z$'s are the unobserved latent variables. Note that the $z$'s induce a clustering over the data points. Our goal is to find the target clustering, with the Bayesian posterior distribution $\pr(z \, | \, x)$ acting as our prior $\pi$ over clusterings for structural QBC.

To sample from this posterior, we used the collapsed Gibbs sampler. Although this Markov chain can take exponential time to mix~\cite{TD14}, we found that running structural QBC with these samples often led to fast convergence to the underlying clustering.

We ran experiments on the {\tt wine} and {\tt iris} datasets from the UCI machine learning repository~\cite{L13}. In each of our experiments, we compared the robust structural QBC strategy (denoted `Robust SQBC') against two baseline strategies: 
\begin{itemize}
\item 'Random': feedback was provided on randomly sampled pairs of points
\item 'Vanilla': no feedback at all, just an unconstrained Gibbs sampler. 
\end{itemize}
The target clustering in each of the datasets was the one induced by the labels of the data points. Thus, in both of these datasets, the space of structures was clusterings containing at most 3 clusters.

For all strategies, we measured the clustering distance from the current clustering to the target clustering, i.e. the fraction of pairs of points the clusterings disagree on, and the log-posterior probability after every pass of the Gibbs sampler. For the feedback strategies, a constraint was added every 50 passes of the Gibbs sampler. The results are displayed in Figures~\ref{fig: wine} and~\ref{fig: iris}.

As can be seen in both of these experiments, clustering error does not perfectly match log-posterior probability. In the \texttt{iris} experiments, for example, we see that the SQBC model does converge upon the target clustering but its log-posterior probability is still much lower than that of the clusterings found by the pure Gibbs sampler. Thus, even though the prior distribution does not put much weight upon the target, the SQBC strategy is still able to find the target clustering after a few rounds of interaction.

\begin{figure}
\begin{center}
\includegraphics[width=5.5in]{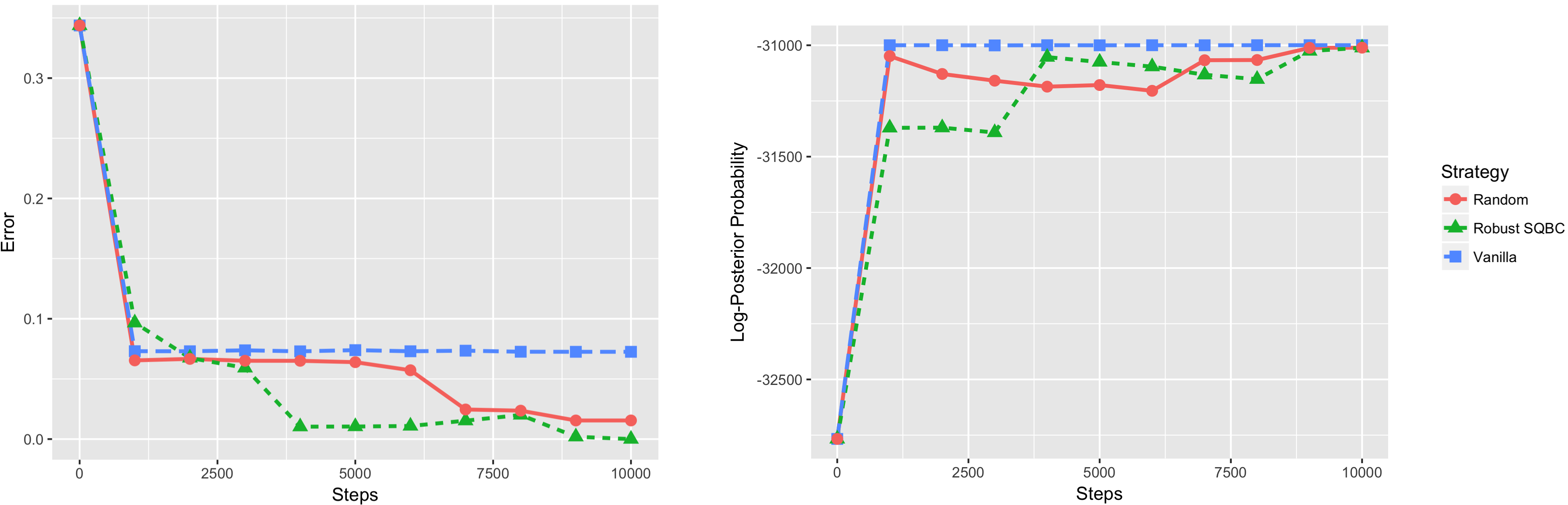}
\end{center}
\caption{Mixture of Bernoullis experiments on MNIST dataset. The hyperparameters $\alpha$, $\beta$, and $\gamma$ were all set to 1.}
\label{fig: mobs plots}
\end{figure}

\paragraph{Mixture of Bernoullis} Consider the following Bayesian generative model for a mixture of $k$ Bernoulli product distributions:
\begin{itemize}
	\item Weight vector $(w_1, \ldots, w_k)$ is drawn from a symmetric Dirichlet distribution with parameter $\alpha > 0$
	\item Bias variables $a_j^{(i)} \in [0,1]$ are drawn i.i.d. from $\mbox{Beta}(\beta, \gamma)$ for $i=1,\ldots, k$ and $j=1,\ldots, d$ 
	\item For each data point $i=1,\ldots,n$:
	\begin{itemize}
		\item $z_i \in \{ 1, \ldots, k\}$ is drawn from $\mbox{Categorical}(w_1, \ldots, w_k)$
		\item $x^{(i)}_j \in \{ 0, 1\}$ is drawn from $\mbox{Bern}(a^{(z_i)}_j)$ for $j=1,\ldots, d$
	\end{itemize}
\end{itemize}
Here, $x^{(1)}, \ldots, x^{(n)}$ are the observed $d$-dimensional binary-valued vectors, the $a$'s are unobserved $d$-dimensional real-valued vectors, and $\alpha$, $\beta$, and $\gamma$ are known hyperparameters.

We ran experiments on a binarized version of the MNIST handwritten digit dataset~\cite{LBBH98}. Here we randomly sampled 50 images from each of the 0, 1, and 2 classes and sought to recover the clusters induced by these classes. The results are presented in Figure~\ref{fig: mobs plots}.

\subsection{Linear separators}

We also considered the classical active learning setting of linear separators. In all our experiments, we used QBC with a spherical prior distribution $\N(0, I_d)$ and the squared-loss posterior update. 

\paragraph{Noisy linear simulations} When learning a linear separator under classification noise, there is a true linear separator $h^* \in \R^d$. When a point $x\in\R^d$ is queried, we observe 
\[ y = \begin{cases} 
\mbox{sign}(\langle h^*, x \rangle) &\text{ with probability } 1-p \\
-\mbox{sign}(\langle h^*, x \rangle) &\text{ with probability } p   \end{cases}\] 

In our simulations, we used various settings of both the noise level $p$ and the aggressiveness of the posterior update $\beta$. In the low noise setting, we found that setting $\beta$ large appears to be appropriate. But as the noise level grows, using a smaller $\beta$ appears to be advantageous. In all settings, however, QBC outperforms random sampling. Figure~\ref{fig: linear plots} shows some results of our simulations. The rest appear in the appendix.

\begin{figure}[t!]
    \centering
    \begin{subfigure}
        \centering
        \includegraphics[width=6in]{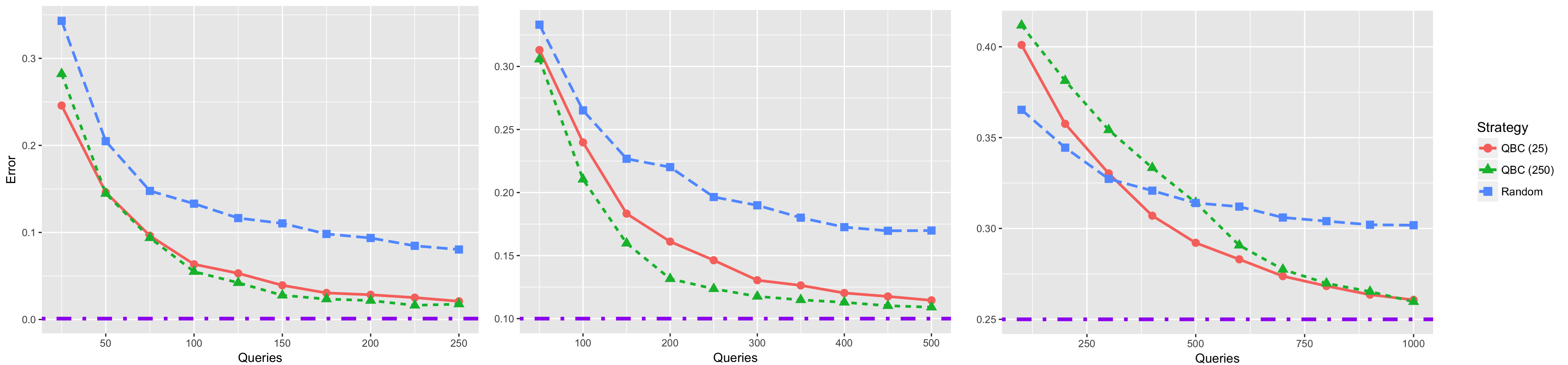}
        \caption{Simulations under different settings of the classification noise $p$. The dashed purple line is the level of classification noise.  In the legend, QBC run with posterior update $\beta$ is shown as `QBC $(\beta)$.' \label{fig: linear plots}}
        \vspace{1.5em}
    \end{subfigure}%
    \begin{subfigure}
        \centering
        \includegraphics[width=5.5in]{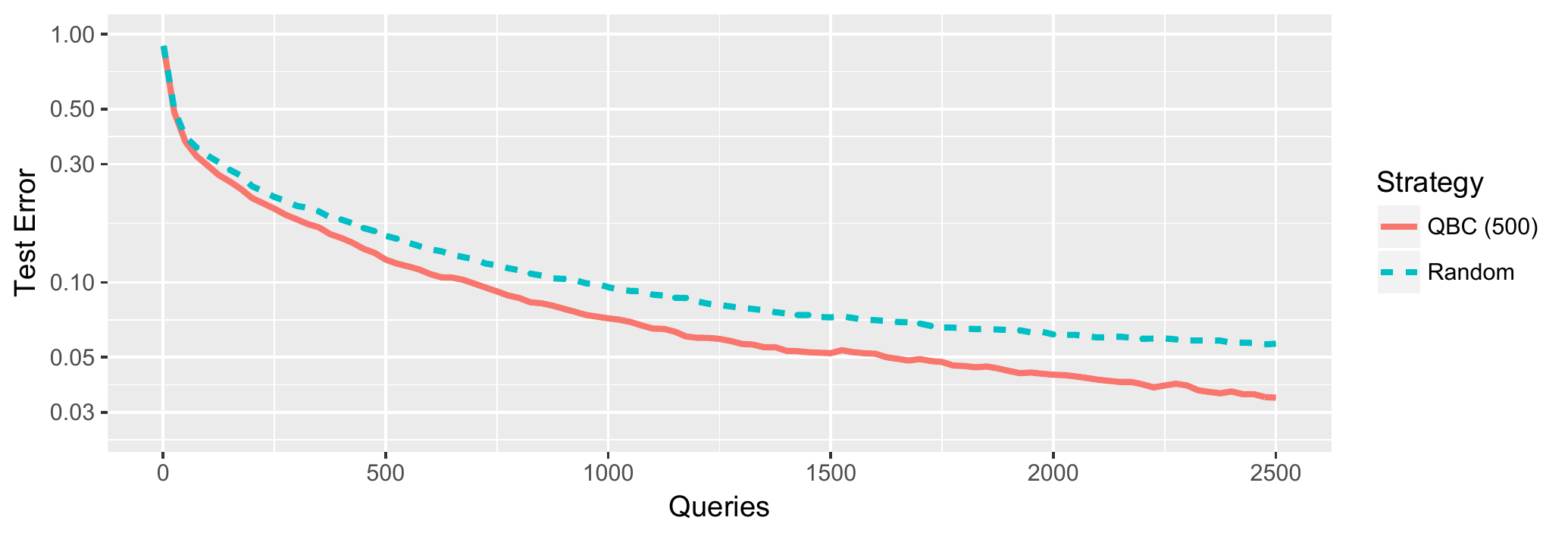}
        \caption{Kernel experiments on MNIST. As in Figure~\ref{fig: linear plots}, QBC run with posterior update $\beta$ is shown as `QBC $(\beta)$' in the legend. Note that the test error axis is log-scale.\label{fig: MNIST plots}}
    \end{subfigure}
\end{figure}

\paragraph{MNIST kernel experiments} We used the full MNIST handwritten digit dataset for our linear classification experiment. For this multiclass classification task, we used interactively trained a collection of ten one-vs-all classifiers using the kernelized squared-loss approach outlined in Section~\ref{sec:kernels}. We used an RBF kernel $K(x, x') = \exp(-\gamma\|x - x'\|^2)$ with the choice of $\gamma = 0.001$. Because we are observing the true labels, we found that large values of $\beta$ worked well for this task. Figure~\ref{fig: MNIST plots} shows the results in this setting. In particular, the experiments show the error obtained after 2500 random labels was obtained after just 1250 actively-sampled labels.

\section{Related work}
\label{sec:related}

In this section, we discuss two areas of work related to the current manuscript: active learning and interactive clustering.

\paragraph{Active learning.} In pool-based active learning, a learner starts with a large collection, or pool, of unlabeled data and adaptively queries points for their labels. Several variations have been studied.

One variant can be thought of as {\em nonparametric active learning}, in which the learning algorithm requests a few of the labels and then fills in all the rest. The objective is to accurately label the entire pool of data while making relatively few queries. One strategy for this problem is to build a neighbor graph over the unlabeled data points and to use this graph to propagate labels from any points that get queried to the remaining points~\cite{ZLG03, CGV09, DNZ15}. In another strategy, a hierarchical partition of the data points is created, and the learner moves down the tree, randomly querying points in each node until it is confident that it has found pure clusters~\cite{DH08}. In general, the success of methods in the nonparameteric setting usually depends on some smoothness property of the underlying data distribution~\cite{CN08, KUB15}.

The other major variant of active learning uses queries to choose a concept from some fixed concept class $\mathcal{H}$. Active learning strategies for this setting range from {\em mellow}~\cite{CAL94, DHM07, BBZ07}, where a label is requested for any point whose label cannot be inferred from those already queried, to {\em aggressive}, where maximally informative points are sought for querying. In the latter category, informativeness can be quantified in a variety of ways. In some strategies, like query-by-committee~\cite{SOS92,FSST97} and generalized binary search~\cite{D04, GB09, N11}, informative points are those that eliminate many hypotheses. Other strategies, like the splitting-index algorithm~\cite{D05} and the related DBAL algorithm~\cite{TD17}, query points that eliminate very different hypotheses.  

The work presented here bears resemblance to both of these settings. On the one hand, our definition of structure is general enough to accommodate arbitrary labelings of a data set, as in nonparametric active learning. However, the algorithms considered here are much more closely related to those in the parametric setting. In particular, the algorithms considered in Section~\ref{sec:sqbc} are clear generalizations of the query-by-committee algorithm, while the robustified algorithm presented in Section~\ref{sec:expts} is in some sense an interpolation between the generalized binary search algorithm and the DBAL algorithm. However, where this work most clearly departs from the classical active learning literature is the change from question-answering to partial correction feedback.

The analysis we present in this paper is closely related to that of the original query-by-committee algorithm~\cite{FSST97}; in particular, it uses a very similar notion of uncertainty. Our results are also very close in spirit to those obtained for generalized binary search~\cite{N11}. In fact, the latter work is able to characterize the query complexity of binary active learning using certain geometric quantities, and it is an interesting open problem whether something similar can be done in our structural setting.

\paragraph{Interactive clustering.} Another area relevant to this work is interactive clustering, which has also developed along several fronts.

One interactive clustering model considered in the literature allows users to split and merge clusters in the algorithm's current clustering until the target clustering has been found. In order to limit the amount of work required of the user, split requests are not allowed to specify how a cluster is split, only that a cluster should be split, and the algorithm must best decide how to go about splitting the chosen cluster. Under certain assumptions on the user's feedback and the target clustering, various algorithms have been shown to recover the target clustering while needing relatively few rounds of interaction with the user~\cite{BB08, AZ10, ABV14}.

Another approach to interactive clustering is the {\em clustering with constraints} framework. In this model, a user provides constraints that the target clustering satisfies, and the algorithm attempts to find a clustering that both satisfies these constraints as well as optimizes some cost function. In the flat clustering setting, the constraints are very often must-link/cannot-link pairs~\cite{WC00}; while in the hierarchical clustering setting, there has been work on constraints that are ordered triplets indicating which points are closer in tree distance~\cite{VD16}. In the case where the target flat clustering optimizes the k-means cost function, it has been shown that relatively few queries can transform an NP-hard unsupervised learning problem into a tractable interactive learning problem~\cite{AKB16}.

As discussed in Section~\ref{sec:structure-learning}, the constraints in both the flat clustering and hierarchical clustering models can be written using in our interactive structure learning framework. It is less clear, however, how to incorporate split and merge feedback into our model. It thus remains an interesting open problem to marry fine-grained interaction such as must-link/cannot-link constraints with higher-level feedback such as cluster splits and merges.

\subsection*{Acknowledgments}

Part of this work was done at the Simons Institute for Theoretical Computer Science, Berkeley, during the ``Foundations of Machine Learning'' program.

\bibliographystyle{plain}
\bibliography{references}

\newpage
\appendix

\section{Deferred Proofs}

\subsection{Proof of Lemma~\ref{lem:gaussian-posterior}}

\GaussianPosteriorLemma*
\begin{proof}
By expanding the form of $\N(\widehat{\mu}, \widehat{\Sigma})$, we first see
\begin{align*}
\N(w \, | \, \widehat{\mu}, \widehat{\Sigma}) &\propto \exp\left(- \frac{1}{2}(w-\widehat{\mu})^T \widehat{\Sigma}^{-1} (w-\widehat{\mu})  \right) \\ \displaybreak[3]
&= \exp\left(- \frac{1}{2} w^T \widehat{\Sigma}^{-1} w - \frac{1}{2}\widehat{\mu}^T \widehat{\Sigma}^{-1} \widehat{\mu} + w^T \widehat{\Sigma}^{-1} \widehat{\mu} \right)\\ \displaybreak[3]
&\propto \exp\left(- \frac{1}{2}w^T \widehat{\Sigma}^{-1} w +   w^T \widehat{\Sigma}^{-1}(2 \beta \widehat{\Sigma} \Phi^T y)  \right) \\ \displaybreak[3]
&= \exp\left(- \frac{1}{2}w^T( 2 \beta \Phi^T \Phi + \frac{1}{\sigma_o^2} I ) w +  2 \beta  w^T \Phi^T y \right) \\ \displaybreak[3]
&= \exp\left(- \beta\left(w^T \Phi^T \Phi w -  2 w^T \Phi^T y \right) - \frac{\|w\|^2}{2\sigma_o^2} \right)
\end{align*}
On the other hand, we have
\begin{align*}
\pi_t(w) &\propto \exp\left(- \beta \| y - \Phi w \|^2 - \frac{\|w\|^2}{2\sigma_o^2}  \right) \\
&= \exp\left(-\beta \left(y^Ty + (\Phi w)^T \Phi w - 2 (\Phi w)^T y \right)  - \frac{\|w\|^2}{2\sigma_o^2} \right) \\
&\propto \exp\left(- \beta\left( w^T \Phi^T \Phi w - 2 w^T \Phi^T y \right) - \frac{\|w\|^2}{2\sigma_o^2} \right)
\end{align*}
Thus $\pi_t = \N(\widehat{\mu}, \widehat{\Sigma})$.
\end{proof}

\subsection{Proof of Lemma~\ref{lem:correction-feedback-shrinkage}}

\begin{lemma}
\label{lem:shrinkage-posterior-mass}
There exists a constant $c>0$ such that the random variable $U_t =  u(q_{t}; \pi_{t-1})$ satisfies \[ \E[U_t \, | \, \F_{t-1}] \geq c \, \pi_{t-1}(g^*) (1-\pi_{t-1}(g^*)) \] for any round $t$ where $q_t$ is the query shown to the user at time $t$ and the expectation is taken over the randomness of structural QBC.
\end{lemma}
\begin{proof}
Note that by the structural QBC strategy, the probability that $q \in \Q$ is selected is proportional to $\nu(q) \, u(q; \pi_{t-1})$. This implies
\[ \E[U_t \, | \, \F_{t-1}] = \frac{\E_{q\sim \nu} [ u(q; \pi_{t-1})^2]}{\E_{q\sim \nu} [ u(q; \pi_{t-1})]} \geq  \E_{q\sim \nu} [ u(q; \pi_{t-1})]. \]
Now take $\nu_o = \min_{q \in \Q} \nu(q)$ and $A_o = \max_{q \in \Q} |A(q)|$. Then for any $g \in \G$ s.t. $g \neq g^*$, there exists some $a \in \A$ such that $g(a) \neq g^*(a)$ which implies 
\[ \E_{q\sim \nu}[d(g,g^*;q)] = \sum_{q \in \Q} \frac{\nu(q)}{|A(q)|} \sum_{a \in A(q)} \ind[g(a) \neq g^*(a)] \geq \frac{\nu_o}{A_o}. \]
Then we have
\begin{align*}
\E[U_t \, | \, \F_{t-1}]  &\geq \E_{q\sim \nu} [ u(q; \pi_{t-1})] \\
&= \E_{g,g' \sim \pi_{t-1}}\left[\E_{q \sim \nu} [d(g,g';q)] \right] \\
&= \sum_{g, g'} \pi_{t-1}(g) \pi_{t-1}(g') \left[\E_{q \sim \nu} [d(g,g';q)] \right] \\
&\geq \pi_{t-1}(g^*) \sum_{g \neq g^*} \pi_{t-1}(g) \E_{q \sim \nu} [d(g,g^*;q) ] \\
&\geq \frac{\nu_o}{A_o} \pi_{t-1}(g^*) (1-\pi_{t-1}(g^*)) .
\end{align*}
\end{proof}

\CorrectionFeedbackShrinkage*
\begin{proof}
Suppose that structural QBC shows the user query $q_t$. Let $\epsilon > 0$ and $U_t = u(q_{t}; \pi_{t-1})$ be the random variable denoting the uncertainty of $q_t$. For an atom $a \in A(q_t)$, we say that $a$ is \emph{known} if
\[ \pi_{t-1}(\{g : g(a) \neq g^*(a)\}) < \epsilon \, U_t. \]
And we say that $a$ is \emph{unknown} otherwise. By a union bound, we have
\[ \pi_{t-1}(\{g : g(a) \neq g^*(a) \text{ for some known } a \in A(q_t) \}) \leq \epsilon\, U_t |A(q_t)|. \]
By Lemma~\ref{lem:shrinkage-uncertainty}, we have that for any $a \in \A$
\[ \pi_{t-1}(\{g : g(a) \neq g^*(a)\}) \geq \frac{1}{2} u(a ; \pi_{t-1}). \]
Moreover, there is \emph{some} unknown atom $a^* \in A(q_t)$ with $u(a^*; \pi_{t-1}) \geq U_t$, implying
\[ \pi_{t-1}(\{g : g(a^*) \neq g^*(a^*)\}) \geq \frac{1}{2} u(a^* ; \pi_{t-1}) \geq \frac{U_t}{2}.  \]
So with probability at least $U_t \left(\frac{1}{2} - \epsilon |A(q_t)| \right)$, a random draw from $\pi_{t-1}$ gets all the known atoms correct and some unknown atom incorrect. Then conditioned on this event occurring, the user has probability at least $p_o$ of correcting some unknown atom. So taking $\epsilon = \frac{1}{4 A_o}$ where $A_o = \max |A(q)|$, we have
\begin{align*}
\E\left[ 1- \gamma_t \, | \, \F_{t-1}  \right] &= \E\left[\pi_{t-1}(\{g : g(a_t) \neq g^*(a_t)\}) \, | \, \F_{t-1} \right] \\
&=\E\left[  \E\left[\pi_{t-1}(\{g : g(a_t) \neq g^*(a_t)\}) | u(q_{t} ; \pi_{t-1}) = U_t \right] \, | \, \F_{t-1} \right] \\
&\geq \E \left[ \frac{U_t^2 p_o}{16 A_o} \, \bigg| \, \F_{t-1}  \right] \geq c_o \E[U_t^2 \, | \, \F_{t-1} ] \geq c_o \E[U_t \, | \, \F_{t-1} ]^2
\end{align*}
where $c_o = p_o/(16 A_o)$ is a positive constant. But by Lemma~\ref{lem:shrinkage-posterior-mass}, we know $\E[U_t] \geq c \, \pi_{t-1}(g^*) (1-\pi_{t-1}(g^*))$ for some constant $c > 0$. The lemma follows by substitution.
\end{proof}

\subsection{Proof of Lemma~\ref{lem:var-shrinkage}}

\begin{lemma}
\label{lem:variance-SQBC-query}
Suppose $\G$ is finite. For each round $t$, the query $q_{t}$ under the structural QBC strategy satisfies 
\[ \E[\var(q_t; \pi_{t-1}) \, | \, \F_{t-1}] \geq c \, \pi_{t-1}(g^*) (1-\pi_{t-1}(g^*)). \]
\end{lemma}
\begin{proof}
For a fixed query $q$, the probability that $q$ gets chosen in round $t$ can be written as
\[ \pr(\text{query } q \, | \, \F_{t-1}) =  \frac{\nu(q) \,  \var(q; \pi_{t-1})}{\sum_{q' \in \Q} \nu(q') \, \var(q'; \pi_{t-1})} = \frac{\nu(q) \,  \var(q; \pi_{t-1})}{\E_{q' \sim \nu}[\var(q'; \pi_{t-1})]}.\]

Taking the expectation of $\var(q; \pi_{t-1})$ over this distribution gives us 
\[ \E[\var(q_t; \pi_{t-1}) \, | \, \F_{t-1}] \geq \frac{\E_{q\sim\nu}[\var(q; \pi_{t-1})^2]}{\E_{q\sim\nu}[\var(q; \pi_{t-1})]} \geq \E_{q\sim\nu}[\var(q; \pi_{t-1})] . \]
Now take $\nu_o = \min_{q \in \Q} \nu(q)$ and
\[ d_o = \min_{g \neq g^*} \max_{q \in \Q, a \in A(q)} \frac{\| g(a) - g^*(a)\|^2}{|A(q)|}  . \] 
For any particular $a \in \A$,
\begin{align*}
\var(a; \pi_{t-1}) &= \frac{1}{2} \sum_{g, g' \in \G} \pi_{t-1}(g) \, \pi_{t-1}(g') \, \|g(a) - g'(a)\|^2 \\
&\geq \pi_{t-1}(g^*) \sum_{g \neq g^*} \pi_{t-1}(g)  \, \|g(a) - g^*(a)\|^2 
\end{align*}
Which implies
\begin{align*}
\E_{q \sim \nu} \left[ \var(q; \pi_{t-1}) \right] &= \sum_{q \in \Q} \frac{\nu(q)}{|A(q)|} \sum_{a \in A(q)} \var(a; \pi_{t-1}) \\
&\geq \sum_{q \in \Q} \frac{\nu(q)}{|A(q)|} \sum_{a \in A(q)} \pi_{t-1}(g^*) \sum_{g \neq g^*} \pi_{t-1}(g)  \|g(a) - g^*(a)\|^2 \\
&= \pi_{t-1}(g^*) \sum_{g \neq g^*} \pi_{t-1}(g) \sum_{q \in \Q} \frac{\nu(q)}{|A(q)|} \sum_{a \in A(q)}  \|g(a) - g^*(a)\|^2 \\
&\geq \nu_o d_o \pi_{t-1}(g^*) (1-\pi_{t-1}(g^*))
\end{align*}
\end{proof}

\VarShrinkage*
\begin{proof}
We first relate the variance of an atom to the probability of a mistake on that atom. To this end, recall
\begin{align*}
D \ &= \ \max_{a \in \A} \max_{g, g' \in \G} \|g(a) - g'(a)\|^2 .
\end{align*}
Moreover, since $\G$ and $\A$ are finite, we have the set of realizable values $S = \{ g(a) : a \in \A, g \in \G \}$ is also finite. Let $d_o > 0$ denote the minimum squared distance between any two elements of $S$:
\[ d_o \ = \ \min_{\substack{s, s' \in S : \\ s \neq s'}} \|s - s'\|^2 .\]
For any atom $a \in \A$, we have
\[  \frac{1}{D} {\var(a; \pi_{t-1})} \ \leq \ \pi_{t-1}(\{g : g(a) \neq g^*(a)\}) \ \leq \ \frac{\var(a; \pi_{t-1})}{d_o \, \pi_{t-1}(g^*)} .\]
To see the left-hand inequality, we can work out
\begin{align*}
\pi_{t-1}(\{g : g(a) \neq g^*(a)\}) &\geq \frac{1}{2} u(a ; \pi_{t-1}) \\ \displaybreak[3]
&= \frac{1}{2} \sum_{g, g'} \pi_{t-1}(g) \pi_{t-1}(g') \ind[g(a) \neq g'(a)] \\  \displaybreak[3]
&\geq \frac{1}{2} \sum_{g, g'} \pi_{t-1}(g) \pi_{t-1}(g') \frac{\|g(a) - g'(a)\|^2}{D} \\  \displaybreak[3]
&= \frac{1}{D} \var(a; \pi_{t-1}) 
\end{align*}
To see the right-hand inequality, notice
\begin{align*}
\var(a; \pi_{t-1}) &= \frac{1}{2} \sum_{g, g'} \pi_{t-1}(g) \pi_{t-1}(g') \|g(a) - g'(a)\|^2 \\
&\geq \pi_{t-1}(g^*)  \sum_{g} \pi_{t-1}(g) \|g(a) - g^*(a)\|^2  \\
&\geq \pi_{t-1}(g^*)  \sum_{g} \pi_{t-1}(g) \, d_o \, \ind[g(a) \neq g^*(a)]  \\
&\geq d_o \, \pi_{t-1}(g^*) \pi_{t-1}(\{g : g(a) \neq g^*(a)\})
\end{align*}
Now suppose that structural QBC shows the user query $q_t$. Let $\epsilon > 0$ and $V_t = \var(q_t; \pi_{t-1})$ be the random variable denoting the variance of $q_t$. For an atom $a \in A(q_t)$, we say that $a$ has \emph{low variance} if 
\[ \var(a; \pi_{t-1}) \ < \ \epsilon V_t. \]
And we say that $a$ has \emph{high variance} otherwise. Then by a union bound, we have
\[ \pi_{t-1}(\{g : g(a) \neq g^*(a) \text{ for some low variance } a \in A(q_t) \}) \ \leq \ \frac{\epsilon V_t |A(q_t)|}{d_o \pi_{t-1}(g^*)}. \]
On the other hand, there is \emph{some} atom $a \in A(q_t)$ with at least average variance, so that 
\[ \pi_{t-1}(\{g : g(a) \neq g^*(a)\}) \ \geq \ \frac{\var(a; \pi_{t-1})}{D} \ \geq \ \frac{V_t}{D}.  \]
So with probability at least $V_t \left(\frac{1}{D} - \frac{\epsilon |A(q_t)|}{d_o \pi_{t-1}(g^*)} \right)$, a random draw from $\pi_{t-1}$ gets all the low variance atoms correct and some high variance atom incorrect. Then conditioned on this event occurring, the user has probability at least $p_o$ of correcting some high variance atom. So taking $\epsilon = \frac{d_o \pi_{t-1}(g^*)}{2 D A_o}$ where $A_o = \max |A(q)|$, we have
\begin{align*}
\E\left[\var(a_t; \pi_{t-1})  \, | \, \F_{t-1} \right] &=\E\left[  \E\left[\var(a_t; \pi_{t-1}) | \var(q_t; \pi_{t-1}) = V_t \right] \, | \, \F_{t-1} \right] \\
&\geq \E \left[ \frac{V_t^2 p_o \epsilon}{2D} \, \bigg| \, \F_{t-1}  \right] \geq c_o \pi_{t-1}(g^*) \E[V_t^2 \, | \, \F_{t-1} ] \geq c_o \pi_{t-1}(g^*) \E[V_t \, | \, \F_{t-1} ]^2
\end{align*}
where $c_o = \frac{d_o p_o}{4 D^2 A_o}$ is a positive constant. But by Lemma~\ref{lem:shrinkage-posterior-mass}, we know $\E[V_t \, | \, \F_{t-1}] \geq c \, \pi_{t-1}(g^*) (1-\pi_{t-1}(g^*))$ for some constant $c > 0$. The lemma follows by substitution.
\end{proof}

\section{Convergence rates}
\label{sec:convergence-rates}

In this section, we bound the rate at which structural QBC's posterior concentrates on the target structure $g^*$. Before we do so, we first introduce some concepts related to the informativeness of feedback.

\subsection{Quantifying the informativeness of feedback}

For an atomic question $a$ and corresponding answer $y$, let $\G_{{a}, {y}} = \{g \in \G: g(a) = y\}$. Define the \emph{shrinkage} of posterior $\pi'$ due to atomic question $a$ to be
$$ S(\pi, a) = 1 - \max_{{y}} \pi(\G_{{a}, {y}})$$
and define the shrinkage of $\pi'$ due to a query $q$ to be the average shrinkage due to $q$'s atoms
$$ S(\pi, q) = \frac{1}{|A(q)|} \sum_{a \in A(q)} S(\pi, a) .$$
This is very similar to the notion of {\it information gain} from the original query-by-committee analysis~\cite{FSST97}, as we explain in further detail at the end of this section. The reason for the modification is that it facilitates the generalizations introduced here: the multiclass setting and noisy feedback. 

In Section~\ref{sec:consistency}, we saw that by making relatively weak assumptions on the specific atoms a user will provide feedback on when presented with a query, we can guarantee the consistency of structural QBC. To get meaningful convergence rates, however, we will need to make a stronger assumption. Specifically, we will require that when presented with a query $q$, the user provides feedback on an atom $a \in A(q)$ whose shrinkage is close to the average shrinkage of the atoms in $A(q)$.

\begin{assump}
When shown $q_t$ and $g_t(q_t)$, the user provides feedback that satisfies
$$\E[S(\pi_{t-1}, a_t)] \ \geq \  S(\pi_{t-1}, q_t)$$
where the expectation is taken over the randomness of the user's choice of $a_t$.
\label{assump:shrinkage}
\end{assump}
Later in this section, we characterize the shrinkage that might be expected in a few cases of interest. For the time being, we will think of it as one of the key quantities controlling the efficacy of active learning, and give rates of convergence in terms of it.

\subsection{Rate of convergence}

In this section, we again consider the setting in which $\G$ is finite. To keep things simple, we think of convergence as occurring when $\pi_t(g^*)$ exceeds some particular threshold $\tau$ (such as $1/2$).

\begin{thm}
Suppose that $\G$ is finite, and that the expert's responses satisfy Assumptions~\ref{assump:0-1-noise} and~\ref{assump:shrinkage}. Suppose moreover that there are constants $0 < \tau, s_o < 1$ such that at any time $t$, if $\pi_t(g^*) \leq \tau$, the shrinkage of the next query is bounded below in expectation as
$$\E [S(\pi_{t-1},q_t) | \F_{t-1}] \geq s_o .$$
Pick any $0 < \delta < 1$. With probability at least $1-\delta$, the number of rounds $T$ of querying before $\pi_T(g^*) > \tau$ can be upper-bounded as follows.
$$ T \ \leq \ 
\left\{
\begin{array}{ll}
\frac{2}{s_o(1-e^{-\beta})} \max \left( \ln \frac{1}{\pi(g^*)}, \frac{4}{s_o (1-e^{-\beta})} \ln \frac{1}{\delta} \right) 
& 
\mbox{if $\lambda = 1$ (noiseless case)} \\
\frac{4}{\beta \lambda s_o} \max \left( \ln \frac{1}{\pi(g^*)}, \frac{8 e^\beta}{\beta \lambda s_o} \ln \frac{1}{\delta} \right) 
& 
\mbox{for any $\lambda$, if $\beta \leq \lambda/2$} 
\end{array}
\right.
$$
\label{thm:0-1-rate}
\end{thm}

\begin{proof}
We will spell out the argument for the noisy case; the other case is similar but slightly simpler. Define
$$ R_t = 1 - \frac{\pi_{t-1}(g^*)}{\pi_t(g^*)} .$$
Using the random variable $\Delta_t$ from Lemma~\ref{lem:0-1-noise}, we have
\begin{align*}
\E [R_t | \F_{t-1}] 
&=
1 - \pi_{t-1}(g^*) \, \E \left[ \frac{1}{\pi_t(g^*)} \bigg\vert \F_{t-1} \right] \\
&=
1 - \pi_{t-1}(g^*) \, \E \left[ \E \left[ \frac{1}{\pi_t(g^*)} \bigg\vert \F_{t-1}, q_t, a_t \right] \bigg\vert \F_{t-1} \right] \\
&= \E [\Delta_t | \F_{t-1} ] \\
&\geq
\frac{1}{2} \beta \lambda \, \E [1-\gamma_t | \F_{t-1}] \\
&\geq 
\frac{1}{2} \beta \lambda \, \E [S(\pi_{t-1}, q_t) | \F_{t-1}] 
\ \geq \ 
\frac{1}{2} \beta\lambda s_o.
\end{align*}
as long as $\pi_{t-1}(g^*) \leq \tau$ throughout.

Pick any time $T$ before $\pi_T(g^*)$ exceeds $\tau$. Then $\E [ R_1 + \cdots + R_T ] \geq \beta \lambda s_o T / 2$. To show that this sum is concentrated around its expected value, we can use a martingale large deviation bound. First we check that each $R_t$ is bounded. Since
$$ e^{-\beta} \pi_{t-1}(g^*) \leq \pi_t(g^*) \leq e^{\beta} \pi_{t-1}(g^*), $$
it follows that $R_t$ lies in an interval of size at most $e^{\beta}$. By the Azuma-Hoeffding inequality~\cite{H63, A67}, if $T$ attains the value in the theorem statement, then with probability at least $1-\delta$,
$$
R_1 + \cdots + R_T 
\ > \ 
\frac{1}{2} \E[R_1 + \cdots + R_T] 
\ \geq \ 
\frac{1}{4} \beta \lambda s_o T 
\ \geq \ 
\ln \frac{1}{\pi(g^*)} .
$$
But this is not possible, since $R_1 + \cdots + R_T$ can be at most $\ln (1/\pi(g^*))$ by the chain of inequalities 
$$ 1 
\ \leq \ 
\frac{1}{\pi_T(g^*)} 
\ = \ (1-R_1)(1-R_2) \cdots (1-R_T) \frac{1}{\pi(g^*)}
\ \leq \  
\exp(-(R_1 + \cdots + R_T)) \frac{1}{\pi(g^*)}.$$
\end{proof}

It is likely that the quadratic dependence on $s_o$, $\lambda$, and $\beta$ can be reduced by a more careful large deviation bound.

\subsection{Shrinkage and uncertainty}

The active learning literature has introduced a variety of different complexity measures that attempt to capture the number of queries needed for learning. These include the {\it disagreement coefficient} \cite{H07b} and the aforementioned {\it information gain} \cite{FSST97}. Although it is possible to give general bounds in terms of these quantities, it has proved quite difficult to compute these complexity measures for all but a few simple cases.

Our notion of shrinkage is a reformulation of the information gain that avoids assuming that the target structure is drawn from the prior distribution, and that accommodates scenarios beyond binary classification. By way of illustration, we will give an example of a simple situation in which the shrinkage can be characterized.

The following lemma demonstrates that under the structural QBC strategy, we can relate a query's shrinkage to its uncertainty.
\begin{lemma}
Suppose the current posterior distribution is $\pi_{t-1}$. Under Assumption~\ref{assump:shrinkage}, the shrinkage $S_t$ of the user's next response has expected value
$$ \E[S_t] \ \geq \ \frac{\E_{q \sim \nu} [u(q; \pi_{t-1})^2]}{\E_{q \sim \nu} [u(q;  \pi_{t-1})]} .$$
\label{lem:shrinkage}
\end{lemma}
\begin{proof}
Fix any query $q \in \Q$ and atom $a \in A(q)$. From Lemma~\ref{lem:shrinkage-uncertainty}, we have 
\[ S_t = \min_{y} \pi(\{g: g(a) \neq y \}) \geq u(a; \pi_t)/2. \]
Now, if the next query is $q_t$, then under Assumption~\ref{assump:shrinkage},
$$ S_t 
\ \geq \ 
\E_{a \sim \mbox{\rm\scriptsize unif}(A(q_t))} \frac{1}{2} u(a; \pi_{t-1})
\ = \ 
\frac{1}{2} u(q_t; \pi_{t-1}).
$$
The lemma follows by taking expectation over $q_t \sim \nu$.
\end{proof}

We now turn to an example where we can bound the shrinkage.

\subsubsection*{Example: Partitioning a hypercube by axis-parallel cuts}

Let $\X$ be the hypercube $[0,1]^p$, and consider axis-aligned bipartitions of $\X$. Any such partition is specified by a coordinate $i$ and a value $v$, and yields clusters
$$ \{x: x_i \leq v\} \mbox{\ and } \{x: x_i > v\} .$$
Let $\pi$ be the uniform distribution over such partitions $\G = [p] \times [0,1]$, so that $\pi(i,v) = 1/p$; and suppose that the data distribution is uniform over $\X$.

We will take atomic queries to be of the form $\{x,y\}$ for $x,y \in [0,1]^p$, where the answer is 1 if they lie in the same cluster, and 0 otherwise. The following lemma shows that the shrinkage of structural QBC queries is always constant in this setting.

\begin{lemma}
Under Assumption~\ref{assump:shrinkage}, the shrinkage $S_t$ of a user's feedback satisfies $\E[S_t] \geq 1/3$.
\label{lem:hypercube-partition}
\end{lemma}

\begin{proof}
Note that if we query $\{x, y\}$ and they are in separate clusters, the version space shrinks to the regions between $x_i$ and $y_i$ on each coordinate $i$; while if they are in different clusters, we get the complement. Either way, the resulting version space is isomorphic to the original $(\G, \pi)$, and hence this is the only case we need consider in computing uncertainty and shrinkage values.

For any $x,y \in [0,1]^p$, the probability that they are separated by a random draw from $\pi$ is exactly 
$$ \sum_{i=1}^p \pr(\mbox{cut coordinate is $i$}) \ |x_i - y_i| \ = \ \frac{\|x-y\|_1}{p} .$$
Thus the uncertainty on a query $\{x,y\}$ is
$$ u(\{x,y\}) \ = \ \pr_{g,g \sim \pi}(\mbox{exactly one of $g,g'$ separates $x$, $y$}) 
\ = \ 2 \cdot \frac{\|x-y\|_1}{p} \left(1 - \frac{\|x-y\|_1}{p} \right) .$$
We will compute the expectation of this over $X = (X_1, \ldots, X_p)$ and $Y = (Y_1, \ldots, Y_p)$ drawn uniformly at random from $[0,1]^p$.

First, a simple one-dimensional calculation shows that
$$ \E [\|X-Y\|_1] \ = \ \sum_{i=1}^p \E |X_i - Y_i|  \ = \ \frac{p}{3} .$$
Likewise,
$$ \E [\|X-Y\|_1^2] 
\ = \ \sum_{i=1}^p \E |X_i - Y_i|^2 + \sum_{i \neq j} (\E |X_i - Y_i|)(\E |X_j - Y_j|)
\ = \ \frac{p}{6} + \frac{p(p-1)}{9} .
$$
Inserting these into the expression for uncertainty, we get
$$ \E[u(\{X,Y\})] \ = \ 2 \left( \frac{\E \|X-Y\|_1}{p} - \frac{\E \|X-Y\|_1^2}{p^2} \right)
\ = \ \frac{4}{9} - \frac{1}{9p} \ \geq \ \frac{1}{3} .$$
We finish by invoking Lemma~\ref{lem:shrinkage} and observing that $\E[S_t] \geq \E[u(\{X,Y\})^2]/\E[u(\{X,Y\})] \geq \E[u(\{X,Y\})]$.
\end{proof}

\subsection{Relation to information gain}

The original analysis of query-by-committee was specifically for active learning of binary classifiers and was based on the notion of {\it information gain}~\cite{FSST97}. Suppose the current posterior distribution over classifiers is $\pi$, and that under this posterior, a specific query $x$ has probability $p$ of having a positive label and probability $1-p$ of having a negative label. Then the information gain is defined to be the entropy of a coin with bias $p$,
$$ I(\pi, x) = H(p).$$
In this same situation, the shrinkage is 
$$ S(\pi, x) = 1 - \max(p,1-p).$$
These two quantities are related by a monotonic transformation. The analysis of QBC's query complexity assumes that the {\it expected} information gain, taken over the random choice of next query, is always bounded below by a constant. In the analysis presented in this section, we assumed the same of the expected shrinkage. In the case of binary classification, these two conditions coincide.

For instance, \cite{FSST97} showed that if the classifiers are homogeneous (through-the-origin) linear separators, and the data distribution is uniform over the unit sphere, then the expected information gain due to a label query is bounded below by a constant. This means that the same holds for the expected shrinkage.

\end{document}